\documentclass[default]{sn-jnl}

\usepackage{amsthm}
\usepackage{amsmath,amssymb,amsfonts}
\usepackage{multicol}
\usepackage{graphicx}
\usepackage{subfig}
\usepackage{graphicx}
\usepackage{tikz}
\usepackage{mathtools}
\usepackage{float}
\newtheorem{definition}{Definition}
\newtheorem{lemma}{Lemma}
\newtheorem{property}{Property}
\usepackage{siunitx}
\usepackage{color}
\usepackage{amsmath}
\usepackage{booktabs}
\usepackage{tabularx}
\usepackage{grffile}
\usepackage[linesnumbered,ruled, noend]{algorithm2e}
\usepackage{textcomp}
\usepackage{flushend}

\usepackage{graphicx}%
\usepackage{multirow}%
\usepackage{mathrsfs}%
\usepackage[title]{appendix}%
\usepackage{xcolor}%
\usepackage{textcomp}%
\usepackage{manyfoot}%
\usepackage{booktabs}%
\usepackage{array,ragged2e}

\begin{document}

\title[ED-Filter]{ED-Filter: Dynamic Feature Filtering for Eating Disorder Classification}

\author*[1]{\fnm{Mehdi} \sur{Naseriparsa}}\email{m.naseriparsa@federation.edu.au}

\author[2]{\fnm{Suku} \sur{Sukunesan}}\email{s.sukunesan@ecu.edu.au}

\author[1]{\fnm{Zhen} \sur{Cai}}\email{zhenc1997@gmail.com}

\author[3]{\fnm{Osama} \sur{Alfarraj}}\email{oalfarraj@ksu.edu.sa}

\author[3]{\fnm{Amr} \sur{Tolba}}\email{atolba@ksu.edu.sa}

\author[4]{\fnm{Saba Fathi} \sur{Rabooki}}\email{saba.fathi.rabooki@student.rmit.edu.au}
 
\author[4]{\fnm{Feng} \sur{Xia}}\email{f.xia@ieee.org}

\affil[1]{Institute of Innovation, Science, and Sustainability, Federation University Australia, Ballarat, VIC 3353, Australia}

\affil[2]{School of Business and Law, Edith Cowan University, Melbourne, VIC 3000, Australia}

\affil[3]{Computer Science Department, Community College, King Saud University, Riyadh 11437, Saudi Arabia}

\affil[4]{School of Computing Technologies, RMIT University, Melbourne, VIC 3000, Australia}


\abstract{Eating disorders (ED) are critical psychiatric problems that have alarmed the mental health community. Mental health professionals are increasingly recognizing the utility of data derived from social media platforms such as Twitter. However, high dimensionality and extensive feature sets of Twitter data present remarkable challenges for ED classification. To overcome these hurdles, we introduce a novel method, an informed branch and bound search technique known as ED-Filter. This strategy significantly improves the drawbacks of conventional feature selection algorithms such as filters and wrappers. ED-Filter iteratively identifies an optimal set of promising features that maximize the eating disorder classification accuracy. In order to adapt to the dynamic nature of Twitter ED data, we enhance the ED-Filter with a hybrid greedy-based deep learning algorithm. This algorithm swiftly identifies sub-optimal features to accommodate the ever-evolving data landscape. Experimental results on Twitter eating disorder data affirm the effectiveness and efficiency of ED-Filter. The method demonstrates significant improvements in classification accuracy and proves its value in eating disorder detection on social media platforms.}




\maketitle

\section{Introduction}\label{sec1}

Eating disorders (ED) have emerged as a profound concern in contemporary mental health research~\cite{barakat2023risk, schneider2023mixed, rivera2022diagnosis}. A growing body of literature attests to the alarming statistic that over 20 percent of Australian adolescents experience some manifestation of eating disorder~\cite{phillipou2020eating, MitchisonJKSANKSP19}. Conditions such as Anorexia Nervosa carry pronounced life-threatening implications, while others may precipitate prolonged mental and physiological complications~\cite{frank2019neurobiology, zhang2021predicting}. Coinciding with the digital age, platforms like Twitter (officially known as X since 2023) have seen a proliferation of communities, termed as \textit{pro-eating disorder} (Pro-ED), which disseminate content promoting and idealizing these disorders~\cite{ARSENIEVKOEHLER16, vall2021impact, HOLLAND16}. Such communities harness the digital milieu to endorse an ethos of thinness and rigorous dieting, inadvertently guiding susceptible populations towards maladaptive eating practices~\cite{prieto2014twitter,ZHANG2020100145}.

However, existing ED classification methods grapple with the multidimensionality and extensive feature set intrinsic to Twitter data. Such complexities not only pose challenges in data extraction but also lead to performance inefficiencies in ED classification, attributable to the high dimensionality and the presence of unreliable features~\cite{XiaAccess2013}. Twitter, emblematic of vast social media platforms, offers a plethora of information that could yield invaluable insights into the dynamics of Pro-ED discourses~\cite{abuhassan2023classification, AnwarFJASTS22}. Yet, the inherent intricacies in this data necessitate an advanced mechanism for its curation, management, and transformation~\cite{ZHOUYRAR19, TIGGEMANN18}. Preliminary identifiers like hashtags $\#proana$ and $\#thinspiration$ might serve as initial markers~\cite{ChancellorPCGC16, alberga2018fitspiration}, but the overarching challenge remains: developing a robust analytical framework that can efficiently navigate the multidimensional space, ensuring accurate ED classification while filtering out unreliable features.

\textbf{Example}: Assuming we have two users who post ED-related messages on Twitter, with the following keywords to check within their posts $\mathcal{K} = \{body, weight, food, meal, exercise, thinspo, suicide, depressed\}$. We then count the occurrences of each keyword within the users' tweets and store these values in separate data rows for each user as follows: (a) $T_1(2,1,0,0,0,0,0,0)$, and (b) $T_2(0,0,5,3,0,0,0,0)$. Assuming we have four groups: $\{\textit{body-image}, food, inspiration, symptoms\}$ for classification. The relationship between the four groups and the keywords in $\mathcal{K}$ are as follows: $\{\textit{body-image}: body, weight\}$, $\{food: food, meal\}$, $\{inspiration: exercise, thinspo\}$, and $\{symptoms: suicide, depressed\}$. Finally, we can classify the $T_1(2,1,0,0,0,0,0,0)$ into $\textit{body-image}$ and $T_2(0,0,5,3,0,0,0,0)$ into $food$.      

The eating disorder data is high-dimensional because the number of keywords $|\mathcal{K}|$ are usually high, i.e., $|\mathcal{K}| \geq 20$. Specifically, the classification of Twitter eating disorder data is extremely unreliable due to the high dimensional features. Some features (keywords) within the data would deteriorate the classification accuracy. To improve the classification accuracy of eating disorders, we propose a dimensional reduction mechanism on high-dimensional data. In this case, feature selection techniques are helpful to discover the most reliable features~\cite{JovicBB15, liu2019shifu2}. It can also reduce the data to address the accuracy problem in Twitter eating disorder data.    


Feature selection is an important pre-processing step for classification tasks, which is frequently adopted to reduce the high dimensional data~\cite{zhang2024feature, tu2023deep, nguyen2023review}. The feature selection improves the performance of the classification models by removing irrelevant features~\cite{JiangZLW19, VallejoTEOC20}. In feature selection, we aim to compute the optimal feature subset which maintains the characteristics of the original data. When the dimensionality of a problem grows, the feature set becomes larger. As a result, computing the optimal feature subset becomes intractable and NP-hard~\cite{milan2017data}. In summary, the feature selection contains four steps: (a)~subset creation, (b)~subset evaluation, (c)~termination condition check, and (d)~result assessment. Feature selection has numerous applications such as in text mining~\cite{ Al-ShammariLNVA17}, data exploration~\cite{NaseriparsaILM18, NaseriparsaILC19, NaseriparsaLIZ19}, intrusion detection~\cite{WANG2020101645}, image classification~\cite{SongZHJ16}, bio-informatics~\cite{HiraG15,  Al-Shammari0LNV18}, and so on.

The feature selection techniques are mainly categorized into two important groups: (a)~the filters~\cite{bommert2020benchmark} and (b)~the wrappers~\cite{liu2022feature}. The wrappers employ the predictive model performance and cross-validation to rank the features and select the most reliable feature set. The wrappers employ a particular classifier to assess the merit of the feature subsets and to compute the optimal feature subset. Therefore, the wrappers are more computationally expensive and slower than the filters. Furthermore, the wrappers guarantee better effectiveness than the filters since they assess the accuracy of the classifier to determine the optimal feature subset. However, the filters apply statistical techniques such as information gain and variance to rank the features. The filters utilize the internal information of a dataset to compute the relevance degree of the features to select the most promising feature subset. These filter methods avoid the optimization process which is done in the wrappers techniques. That is why the filters are faster and less costly in comparison with their wrappers counterparts. However, the filters accuracy in feature ranking is not as reliable as the wrappers. There are some methods that take advantage of combining filtering with wrappers. In such methods the filtering is applied before performing the wrapper to alleviate the computational costs of the wrapper~\cite{AbediniaNH17}. 

A group of feature selection methods is based on Bayesian networks to weight the features for filtering purposes~\cite{TangKH16}. Other groups of feature selection methods utilize the correlation and information theory to compute the weight of the features and to perform the filtering task \cite{zhang2022scstcf, GeZLMMMWZ16, VallejoTEOC20}. For instance, Sun et al. \cite{SunWQXZ19} proposed feature selection for incomplete data by using the Lebesgue and entropy measures. They claimed the method handled mixed and incomplete datasets and maintained the original classification information.

While the wrappers are suitable for static data with reasonable volume, they are impractical in dynamic data environments with huge data volumes. To this end, we propose \textit{ED-Filter} for reducing the feature dimensionality in the Twitter eating disorder classification problem. We design a branch \& bound algorithm that performs an informed search with a termination condition to check the features' eligibility for joining the optimal feature subset, which maximizes the classification accuracy. We have incorporated a greedy strategy into the informed branch and bound search algorithm to accelerate the exploration of ED's reliable feature set. As a result, the \mbox{greedy-based} search method generates a sub-optimal ED feature subset, which leads to more efficient pruning of non-promising branches during the exploration process. This sub-optimal ED feature subset maintains satisfactory effectiveness for the classification of eating disorders while significantly reducing exploration time. Furthermore, we utilize deep-learning techniques to simplify the computing process of the ED feature subset. By combining deep learning with a greedy-based search approach, we focus only on the most promising section of the search space and quickly ignore the less promising parts. We employ a multi-layer perceptron model to design our deep learning technique and learn the appropriate feature subset size; thus, we concentrate on feature subsets of the same sizes. Our multi-layer perceptron model has two hidden layers to identify the most promising part of the search space. We train the model using Twitter data to predict the optimal size of the ED feature subset. This hybrid greedy-based deep-learning search approach effectively initiates the exploration task in the most promising search space; therefore, it is well-suited for the dynamic classification of eating disorders where the Twitter data streams arrive frequently and are modified numerously.  

Our contributions are summarized as follows:

\begin{itemize}
    \item We address the eating disorder classification problem by focusing on the \mbox{high-dimensional} data issue.  
    \item We propose a feature selection method called ED-Filter which utilizes a branch \& bound algorithm to reduce high-dimensional eating disorder data efficiently. 
    \item We devise a hybrid greedy-based deep learning solution based on multi-layer perceptron model to select the most reliable features for eating disorder classification problem on dynamic Twitter data efficiently.   
    \item We conduct experiments on Twitter data to verify the efficiency and effectiveness of our proposed method and to compare it with state-of-the-art feature filtering methods.
\end{itemize}


This paper is organized as follows: firstly, we review the related work in Section \ref{Sec:Related Work}. Then, we define our problem in Section \ref{Sec:ProblemDefinition}. We present our proposed ED-Filter method in Section \ref{Sec:ED-Filter}. In Section \ref{Sec:ExactAlgorithm}, we discuss the exact algorithm which produces the optimal solution. Section \ref{Sec:GreedySuboptimal} presents the greedy-based suboptimal solution. In Section \ref{Sec:HybridDeppLearning} we discuss the hybrid greedy deep-learning method. Section \ref{Sec:ExperimentalAnalysis} presents the experimental analysis and results. Finally, we conclude the paper in Section \ref{Sec:Conclusion}.    

\section{Related Work}
\label{Sec:Related Work}
Social media data analytics is widely used by researchers to address various mental disorder issues~\cite{Kim2021}. For instance, it can be applied to examine drug abuse prescriptions~\cite{Raza2023, Sarker2020}, identify signs of depression~\cite{Babu2022-nz}, combat anxiety disorders~\cite{AHMED2022100066}, assess public sentiment regarding vaccination~\cite{RAHMANTI2022106838}, and explore eating disorders~\cite{abuhassan2023classification, AnwarFJASTS22}. In this section, we first review recent studies related to eating disorders. We then review the feature filtering techniques for high-dimensional eating disorder data.

\subsection{Eating Disorders}
Eating disorders are critical psychiatric problems which promote radical weight management behaviours~\cite{mcclure2023predictors, AndradesARBMVO21}. Marie et al.~\cite{GALMICHE19} reviewed the prevalence of the different EDs and explored their evolution. In their review, they confirmed that EDs are prevalent globally, especially in females. The ED-related behaviors are often promoted by special communities called (Pro-ED) on social media platforms by sharing various ED-related messages and images~\cite{FergusonMGG14}. Ferguson et al.~\cite{FergusonMGG14} showed that these communities exchange the ED-related plans, including dieting, exercise and images of extremely thin bodies. Dodzilo et al.~\cite{Dondzilo2024} conducted a study to investigate the potential association between engagement with appearance/eating-related TikTok content and ED symptoms. They showed how the content engagement on TikTok is linked to targeted exposure, which leads to ED symptoms.
Dodzilo et al.~\cite{DONDZILO2024101923} experimentally identified that the exposure to social media exacerbates the ED symptoms and the participants who limited their social media experienced a reduction in ED symptoms. \mbox{Rodgers et al.}~\cite{Rodgers2024} discussed the importance of social media literacy to combat the harmful effects of appearance-focused photo-based content related to the promotion of radical behavior that results in developing eating disorder symptoms. Therefore, harnessing the content for intervention and developing social media skills among people would be necessary to contain the harmful effects of eating disorder content on social media.


The dynamics and structure of Pro-ED communities have been widely studied by the literature to highlight the complexities and mechanisms in which these communities work on social networking platforms such as Twitter~\cite{ARSENIEVKOEHLER16, ZHOUYRAR19, TIGGEMANN18}. 
\mbox{Arseniev-Koehler et al.}~\cite{ARSENIEVKOEHLER16} studied the Pro-ED profiles on Twitter to explore the social connections between Pro-ED communities. Zhou et al.~\cite{ZHOUYRAR19} applied the Correlation Explanation (CorEx) topic model on Twitter data to identify ED-related topics. They identified twenty topics and group them into eight categories. \mbox{Tiggemann et al.}~\cite{TIGGEMANN18} compared two contemporary ED-related communities, including thinspiration and fitspiration on Twitter. They conducted sentiment analysis on the Tweets collection and showed that fitspiration tweets were more positive in terms of sentiment. \mbox{Anwar et al.}~\cite{AnwarFJASTS22} proposed a comprehensive ED lexicon, called EDBase to facilitate the content analysis tasks within the ED-related conversations by using Twitter data. The proposed EDBase contains a full list of high-quality ED terms with an ED score that is linked to their parent terms. Twitter is known for its micro-blogging property which is popular among young people~\cite{DuitsKBAG23}. The users utilize Twitter by creating a profile and posting microblogs called ``tweets''. A tweet is a short text that other users may share. Moreover, the users can follow other users and build a social network. Twitter platform plays a pivotal role in exchanging information and facilitating online social activity by providing users with the option of posting various files such as images, videos, or web links. The users are allowed to add keywords or hashtags to the existing tweet to create a larger conversation~\cite{ARSENIEVKOEHLER16, ZHOUYRAR19, TIGGEMANN18}. \mbox{Benítez-Andrades et al.}~\cite{Ben_tez_Andrades_2023} proposed a hybrid framework that combines contextual knowledge from multiple types of data sources, including unstructured and structured texts. They utilize the BERT to generate vector embedding from structured data, and employ Wikidata to generate knowledge graph-based embeddings. As a result, they develop a predictive model for a typical classification task. Finally, they evaluated their model by a corpus of eating disorders Tweets.

\subsection{Feature Selection}
Pro-ED communities employ Twitter platform to share experience and encourage radical ED-related behavior by creating ED-related hashtags and posting the corresponding tweets~\cite{di2023methodologies, TIGGEMANN18}. However, Twitter data collection and analysis is challenging due to enormous amount of generated data by the pro-ED communities~\cite{AnwarFJASTS22, AndradesARBMVO21}. Moreover, many machine learning tasks on Twitter eating disorder data focus on classification~\cite{AndradesARBMVO21, ZHOUYRAR19}. The Twitter eating disorder data classification often faces the high dimensional space issue due to the high number of keywords and hashtags.  

Therefore feature selection, a well-established dimensionality reduction technique, is crucial to improve the classification performance~\cite{cinar2023novel, DeviS18}. As mentioned previously, there are two significant feature selection techniques in literature: a)~filters~\cite{bommert2020benchmark} and b)~wrappers~\cite{liu2022feature}. Filter feature selection utilizes the intrinsic characteristics of a dataset, such as information gain and variance, to rank the features, which is quick. For example, Arya and Gupta~\cite{AryaG23} proposed a filter-based feature selection approach based on a combination of four popular filter-based feature selection techniques, including ANOVA, Pearson Correlation Coefficient, Mutual Information, and Chi-Square~(CS). However, the filters often ignore the learning model during the feature selection process; therefore, they may perform unsatisfactorily. On the other hand, wrappers utilize the predictive model performance
to rank the features, which is computationally expensive. For example, Seghir et al.~\cite{SeghirDSC23} proposed a wrapper-based feature selection for medical diagnosis which combines Binary Teaching-Learning Based Optimization (BTLBO) algorithm with the K-Nearest Neighbor (KNN) classifier to explore the optimal features subset. In this paper, our focus is to propose a dynamic feature selection technique to select the most
reliable features for dynamically generated Twitter eating disorder data. Our feature selection approach is a hybrid technique that employs a greedy-based deep learning technique to facilitate the dynamic feature selection in Twitter eating disorder data.

\section{Problem Definition}
\label{Sec:ProblemDefinition}
To create and analyze the eating disorder data for classification, we first collect the Tweets (raw data) from the users who post messages about eating disorders. 

\begin{definition}\textbf{Eating Disorder Raw Data}
Eating disorder raw data from Twitter is presented as $\mathcal{D} = \{D_1,D_2,...,D_n\}$ where $D_i \in \mathcal{D}$ represents the collection of eating disorder-related Tweets that are posted by a user, i.e., $D_1=\{p_1,p_2,...,p_n\}$. Each $p_i \in D_1$ is a post from a particular user that directly contains keywords or hashtags related to eating disorders.  
\end{definition}

Moreover, there are a number of features (dimensions) that convey meaning with respect to eating disorder symptoms. These features are represented by keywords or hashtags within the users' Tweets. Thus, we analyze the Twitter messages and count the number of times each feature appeared within the users' Tweets. Finally, we come up with an eating disorder data structure which is presented as follows:   

\begin{definition}\textbf{Eating Disorder Data Row}
Eating disorder data row from Twitter is presented as $T(F;y)$ where $F=\{f_1,f_2,...,f_n\}$ is a set of eating disorder features and $y$ is the corresponding eating disorder type.  
\label{def:ed_data_row}
\end{definition}

For example, $T_1(5,1,0,3,0,0,0,8,0,0;0)$ is a Twitter data row that contains ten eating disorder features $F_1=\{5,1,0,3,0,0,$ $0,8,0,0\}$ and the corresponding eating disorder type $y_1=0$. To differentiate between multiple features within the Twitter eating disorder data, we should estimate the importance of the individual features. Therefore, we adopt the established TF-IDF technique~\cite{Kang2024} and estimate the importance of the individual features by their frequencies as follows:  

  
 

\begin{table}
\centering
\caption{An Example of Twitter Eating Disorder Data $\mathcal{T}$. Each row represents a data row (as formulated in definition~\ref{def:ed_data_row}), corresponding to one user.}
\scriptsize
\begin{tabular}{ |c|c|c|c|c|c|c|c|c|c|c|c|c|c|c|c|c| }
  
  \hline
     $f_1$ & $f_2$ & $f_3$ & $f_4$ & $f_5$ & $f_6$ & $f_7$ & $f_8$ & $f_9$ & $f_{10}$ & $f_{11}$ & $f_{12}$ & $f_{13}$ & $f_{14}$ & $f_{15}$ & y \\
  \hline
  1 & 0 & 0 & 0 & 0 & 0 & 0 & 0 & 0 & 0 & 0 & 0 & 0 & 0 & 0 & 0  \\
  \hline
  0 & 1 & 0 & 0 & 0 & 0 & 0 & 1 & 0 & 0 & 0 & 0 & 1 & 0 & 0 & 1 \\
  \hline
  26 &	22 & 0 & 0 & 25 & 0 & 0 & 3 & 0 & 0 & 0 & 0 & 29 & 0 & 0 & 2 \\
  \hline
 1 & 39 & 0 & 0 & 0 & 0 & 0 & 0 & 0 & 0 & 0 & 0 & 0 & 0 & 0 & 3 \\
 \hline
 
\end{tabular}
\label{tbl:TwitterDataSamples}
\end{table}

\begin{definition}\textbf{Feature Weight}
For each data row $T$, the feature $f_i \in F$ within the Twitter data row is assigned a weight. This weight is computed by the summation of the frequencies of a set of keywords $\{k_1,k_2,...,k_n\}$ that are appeared in the users tweets as follows:

\begin{equation}
    f_i = \sum_{j=1}^n freq(k_j)
\end{equation}
where $freq(k_j)$ returns the number of times that $k_j$ appeared in the user tweets.  
\end{definition}

The Twitter eating disorder data $\mathcal{T}$ is generated by collection of Twitter eating disorder data rows, i.e, $\mathcal{T}=\{T_1,T_2,...,T_n\}$.   

\textbf{example}. Table \ref{tbl:TwitterDataSamples} presents an example of the Twitter eating disorder data $\mathcal{T}$. It represents four data rows from four users. Our Twitter data is classified in four categories, i.e., $y \in \{0,1,2,3\}$. Each row in Table \ref{tbl:TwitterDataSamples} summarizes the statistics of the eating disorder Tweets that are posted by a particular user in Twitter.

However, the number of features is usually even larger than $10$, i.g., $15$, which is relatively high. This high number leads to deteriorating the classification accuracy with respect to eating disorder types in Twitter data $\mathcal{T}$. To address the high dimensionality problem within the eating disorder data, we apply a well-established feature selection method to extract the most reliable features more effectively.    

\begin{definition} \textbf{Feature Selection}.
Given a feature set $F = \{f_1, f_2,$ $..., f_n\}$, we want to reduce the feature set to $F^{'} = \{f_1, f_2,..., f_m\}$ where $m < n$.
\end{definition}

As mentioned, the main goal of feature selection is to improve the eating disorder classification accuracy. Here, we define our measure of accuracy $\theta$ to assess the feature selection effectiveness. Assume $Y$ is a set of eating disorder types for data rows, i.e., $Y= \{y_1,y_2,...,y_n\}$. Then, the classification accuracy of a set of features $F^\prime \subset F$ concerning the Twitter data $\mathcal{T}$ is defined as $\theta(F^\prime,Y)$. 

\begin{equation}
\theta(F^\prime,Y) = \frac{TP(F^\prime,Y)}{TP(F^\prime,Y) + FP(F^\prime,Y)}    
\end{equation}
where $TP(F^\prime,Y)$ denotes the number of cases that are correctly diagnosed as their respective eating disorder type $Y$ by having the feature set $F^\prime$ while $FP(F^\prime,Y)$ denotes the number of cases that are incorrectly diagnosed as eating disorder type $Y$ by having the feature set $F^\prime$. 

\textbf{Problem Statement}. Suppose $\mathcal{F} = \{F_1,F_2,...,F_n\}$ is the super set that contains all possible feature subsets within the feature space in eating disorder data $T(F,Y)$. Then, we like to find the optimal feature subset $F^* \in \mathcal{F}$ within the eating disorder data such that $\forall F_i \in \mathcal{F}$, $\theta(F^*,Y) \geq \theta(F_i,Y)$. 

Since the eating disorders data from Twitter arrive frequently, the optimal feature subset is subject to change; thus, the system should apply the feature subset extraction process numerously. The feature extraction incurs heavy computational costs to the system. To alleviate the exponential costs, we propose a hybrid deep learning greedy-based feature subset extraction algorithm. The hybrid algorithm skips many unnecessary computations and improves the feature selection performance of the dynamic Twitter eating disorder data streams.

\section{The Proposed ED-Filter Method}
\label{Sec:ED-Filter}

In this paper, we propose FilterBoost that is designed in two phases: (a) the feature ranking phase that ranks the feature with respect to their relevance for the medical diagnosis, and (b) dimensionality reduction based on subset selection from $F_{D}$ in multiple iteration. In each iteration, if the subset reduction leads to improve the medical diagnosis, we update the feature set. Otherwise, we add the reduced subset to the feature set and go to the next iteration.

Fig.~\ref{fig:SystemArchitecture} presents the ED-Filter architecture for reducing the high-dimensional eating disorder data. The newly arrived data from Twitter is stored in the database dynamically. Due to frequent modification of the eating disorder classification data, we employ a feature selection technique that combines the wrappers with filters called ED-Filter. That is because the wrappers incur intensive computational costs and long response times, whereas the filtering methods are computationally less expensive. Thus, we combine them in ED-Filter to efficiently work on the evolving Twitter eating disorder data to extract reliable features. After the feature selection method extracts the reliable features, we remove the irrelevant features. Therefore, we utilize reliable features to generate an effective eating disorder diagnosis model for classification tasks.

\begin{figure}[h]
 \centering
	  \includegraphics[scale=0.8]{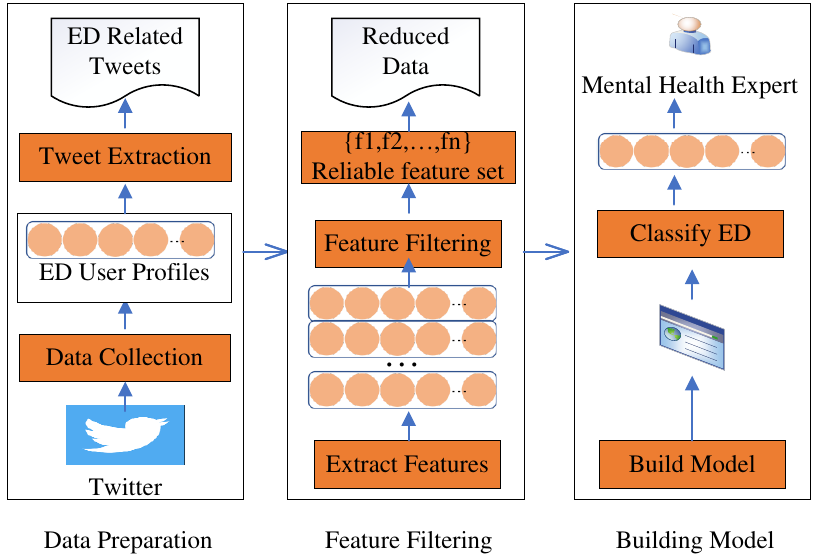}
		\caption{The architecture of ED-Filter.}
		\label{fig:SystemArchitecture}
\end{figure}

The ED-Filter is a feature selection method designed for the high-dimensional eating disorder data of Twitter. The ED-Filter utilizes informed search and deep learning to explore the reliable feature subset with the best accuracy. 

Feature filtering techniques utilize statistical methods such as information gain~\cite{venkatesh2019review} to rank the features of high-dimensional data. The feature ranking is necessary for detecting the most relevant features. Thus, the low-ranking features below a threshold are considered irrelevant and removed. Feature filtering reduces the data dimensionality and improves the eating disorder diagnosis by removing non-promising features. However, the feature filtering methods are not accurate. For example, they sometimes remove reliable features that affect the classification accuracy negatively. To address this weakness of feature filtering techniques, we take the benefit of wrappers and propose a boosting mechanism called \textit{ED-Filter} (Eating Disorder Filter) to improve the feature selection accuracy. 

\begin{figure}[h]
 \centering
	  \includegraphics[scale=0.8]{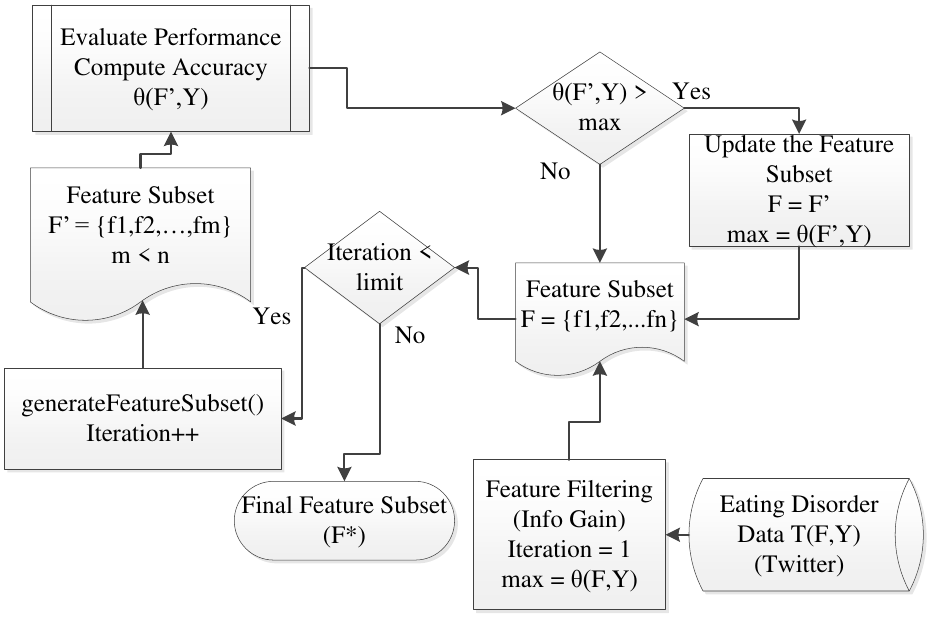}
		\caption{The steps in ED-Filter.}
		\label{fig:FilterBoost}
\end{figure}

Figure~\ref{fig:FilterBoost} presents the steps of the ED-Filter method. The ED-Filter contains two phases: (a)~ranking the features and (b)~iteratively assessing subsets of eating disorder data features to achieve the best classification accuracy. From the Figure, we use the Information Gain filtering technique to rank the features of the eating disorder data in the first step. Then, we generate a feature subset $F^\prime \subset F$. ED-Filter computes the eating disorder classification accuracy $\theta(F^\prime,Y)$ to identify the merit of the feature subset $F^\prime$. If the feature subset merit is better than the previous step, we update the feature set. Otherwise, we keep the feature set and choose another subset to assess. These steps continue until we have reached the iteration limit when we choose the current feature set to build the eating disorder classification model. 

The iteration limit is a tuning parameter that directly affects the performance when the model works on dynamic incoming data. For example, when the initial data contains $n$ features, the maximum iteration limit is $(2^n - 1)$. That is because all feature subsets are evaluated in terms of their accuracy. Although the classic feature filtering methods are suitable for a static environment, Twitter eating disorder data are produced and modified frequently. In this case, we should adapt the feature filtering process to work in a dynamic environment. Thus, we should consider the performance of feature filtering for the dynamic eating disorder data. To this end, we propose an informed branch \& bound search algorithm that performs a search on the feature space of the eating disorder data to find the optimal features. Moreover, we further improve the performance by using a greedy-based branch \& bound algorithm that retrieves a sub-optimal feature subset. The greedy method utilizes the deep learning technique to skip many feature subsets; thus, it only focuses on the most promising feature subsets.   

\subsection{Ranking the Features}
In the first phase of the proposed ED-Filter, we apply a feature filtering technique to rank the data features. The feature filtering methods utilize statistical techniques to rank the features against the data. In this paper, we employ Information Gain, an established filtering technique for feature ranking. However, our proposed ED-Filter is flexible, and we can adopt other statistical techniques instead. The Information Gain filtering utilizes the Entropy concept to compute the rank of features. The Entropy measures the feature unpredictability or surprise level. For example, when the Entropy increases, the feature produces more uncertain or surprising outcomes. Equation \ref{eq:Entropy} presents the entropy formula for the class feature $Y$. In the following formula, $p(y)$ is the probability density function for the class feature $Y$. 

\begin{equation}
H(Y) = -\sum_{ y \in Y} {p(y) \log_2(p(y))}. 
\label{eq:Entropy}
\end{equation}

Equation \ref{eq:EntropyCond} presents the entropy of class feature $Y$ after observing the feature $X$. In the following formula, $p(y|x)$ is the conditional probability of $y$ given $x$. 

\begin{equation}
H(Y|F) = -\sum_{ f \in F} {p(f) \sum_{y \in Y}p(y|f)\log_2(p(y|f))}. 
\label{eq:EntropyCond}
\end{equation}

The Information Gain (IG) metric computes the dependency level between two features. For our filtering method, the IG computes the statistical dependence between the class feature $Y$ and other feature, i.e., $F_i \in \mathcal{F} = \{F_1, F_2,..., F_n\}$ to rank the feature. That means the IG score of $F_i \in \mathcal{F}$ is the average reduction in uncertainty about a value from $F_i$ that results from learning a value from class feature $Y$. Equation~\ref{eq:IG} presents the IG formula where $H(Y)$ is the entropy for class feature $Y$ and $H(Y|F)$ is the conditional entropy of $Y$ given the feature $F$.      
\begin{equation}
IG(Y;F) = H(Y) - H(Y|F). 
\label{eq:IG}
\end{equation}

\subsection{Iterative Feature Reduction}
In the second phase of the ED-Filter method, we take the benefit of wrappers to iteratively remove a subset of irrelevant features from the feature set, i.e., $F^{\prime\prime} =\nobreak F \setminus F^{\prime}$. We assess the classification performance of a feature set to determine the irrelevant features in each round. For example, after removing $F^\prime$ from $F$, we assess the eating disorder classification model on the new feature set $\theta(F^{\prime\prime},Y)$. If the performance improves (this means that the removed features were irrelevant), we update the feature set with the newly reduced feature set, i.e., $F = F^{\prime\prime}$. Otherwise, we keep the feature set and move to the next iteration. The performance of the eating disorder classification model is determined by classification accuracy. 


\subsection{Feature Cardinality Detector}
The feature selection process continues until we reach the iteration limit. If we leave the iteration limit to a large number, the feature selection performance deteriorates on Twitter data because the data is produced and modified frequently. To effectively reduce the iteration limit, we design a feature cardinality detector. The detector predicts the cardinality of the feature subset so that the filtering process focuses on producing the best feature subset with a given cardinality. To design and develop the cardinality detector, we employ a neural network model, a multi-layer perceptron (MLP), to detect the cardinality of the feature subset. We prepare a training set for the neural network to build a model for detecting the cardinality of the feature subset to speed up the feature selection process. The training data contain the full feature set $F$ and the number of selected features $|F^\prime|$, which is used to build the deep learning model to predict the feature subset cardinality for the ED-filter. We embed the cross-validation technique in developing the cardinality detector component. For instance, we break down the Twitter data into smaller and similar chunks and feed the neural network with these chunks. We iterate the training phase and replace the validation set until the model is trained and validated by all chunks.     

\section{Branch \& Bound Search Algorithm}
\label{Sec:ExactAlgorithm}
The brute force solution for detecting the most reliable features on dynamic Twitter data generates all possible feature subsets; then, it ranks them based on the accuracy and identifies the optimal subset $\mathcal{F}^*$ with the maximum accuracy. However, this solution is practically impossible due to combinatorial explosion. That is because we should generate a feature power set which contains up to $2^n - 1$ subsets. Moreover, we should measure the accuracy of each element in the power set, which is exponential. Therefore we propose an informed branch and bound search to skip non-promising feature subsets quickly. To design the informed branch and bound search, we need to estimate the upper bound of the classification accuracy. Since we employ information gain for filtering the feature space, we identify the relationship between classification accuracy $\theta(F,Y)$ and information gain $IG(Y;F)$ as follows:  


\begin{equation}
    \log{(n)} - H - (1 - \theta)\log{(n')} \leq IG(Y; F).
\label{eq:upperboundshort}    
\end{equation}

In this equation, $H$ denotes $H_2(\theta(F,Y))$, which is the binary entropy function. $\theta(F,Y)$ is simplified as $\theta$, represent the classification accuracy of feature set $F$ with given class $Y$. $(1 - \theta)\log{(n')}$ represents the product of $(1 - \theta)$ and the logarithm of $n'$, where $n'$ is a stand-in for $(n - 1)$. $IG(Y; F)$ represents the information gain score of feature set $F$ and class $Y$. We then use equation~\ref{eq:upperbound} to find the upper bound for accuracy given the mutual information for a feature Y as follows:    

\begin{lemma}
The upper bound value of the classification accuracy for a set of features $F=\{f_1,f_2,...,f_n\}$ based on its mutual information is computed as follows: 
    \begin{equation}
    \overline{\theta}(F,Y) = \frac{IG(Y;F) - \log(n) + 1 }{\log(n-1)} + 1.   
    \label{eq:upperbound}    
    \end{equation}
\end{lemma}

\begin{proof}
$\theta(F,Y) \leq \frac{IG(Y;F)  - \log(n) + H_2(\theta(F,Y)) }{\log(n-1)} + 1$ Since the maximum value of a binary entropy $\overline{H_2}(\theta(F,Y))$ = 1, then we replace the binary entropy $H_2(\theta(F,Y))$ with its upper bound value $1$ in the equation: $\theta(F,Y) \leq \frac{IG(Y;F)  - \log(n) + 1 }{\log(n-1)} + 1$, hence the upper bound becomes $\overline{\theta}(F,Y) = \frac{IG(Y;F) - \log(n) + 1 }{\log(n-1)} + 1$.
\end{proof}

To ensure the exact algorithm correctness, the upper bound estimation for classification accuracy $\overline{\theta}(F,Y)$ should never underestimate the real classification accuracy. To this end, we prove the upper bound estimation of classification accuracy never underestimates the real accuracy; therefore, it is admissible. Assume $\mathcal{F} = \{F_1,F_2,...,F_n\}$ is the super set of feature subsets.  

\begin{lemma}
Given $\forall F^\prime \in \mathcal{F}$, $\overline{\theta}(F,Y) \geq \overline{\theta}(F^\prime,Y)$. The $F^*=\{f_1,f_2,...,f_n\}$ is the optimal feature subset if (a)~the stop condition $\theta(F^*,Y) \geq \overline{\theta}(F,Y)$ satisfies and (b)~the upper bound $\overline{\theta}(F,Y)$ is admissible. 
\label{lem:admissible}
\end{lemma}

\begin{proof}
The upper bound estimation of classification accuracy  $\overline{\theta}(F,Y)$ is admissible; thus we have, $\overline{\theta}(F,Y) \geq \theta(F,Y)$. Also, $\overline{\theta}(F,Y)$ is the maximum upper bound estimation among other upper bound estimations of feature subsets. Using the admissibility, $\forall F^\prime \in \mathcal{F}, \theta(F^\prime,Y) \leq \overline{\theta}(F,Y) \leq \theta(F^*,Y)$. Therefore, $F^*$ is the optimal feature subset.      
\end{proof}

Algorithm \ref{alg:branchbound} presents the branch \& bound algorithm to compute the optimal feature subset $F^*$. We initialize the optimal feature subset $F^*$, create a max heap $\mathcal{H}$, and initialize it to store the search information in line 1. The max heap $\mathcal{H}$ contains a collection of entries, i.e., $e \in \mathcal{H}$. Each entry $e$ has three information components: (a) $e.F$ is the feature subset, (b) $e.\theta$ contains the classification accuracy of the entry feature subset $\theta(e.F,Y)$, and (c) $e.\overline{\theta}$ which is the upper bound estimation of the classification accuracy of the entry feature set $\overline{\theta}(e.F,Y)$. In lines 2-6, we kick off the search by initializing $\mathcal{H}$ with the first elements. In lines 7-8, we pop the top element from $e$. If the termination condition $e.\overline{\theta} < \theta^{min}$ is satisfied, we stop the search as pseudocoded in lines 9-10. Otherwise, we expand the feature subset $e.F$ by adding more features from the feature set $F$ and generate the new feature subset $e^\prime.F$ in line 12. Lines 13-14 compute the new expanded entry classification accuracy $e^\prime.\theta$ and upper bound accuracy $e^\prime.\overline{\theta}$ respectively. Then, we push the new expanded entry $e^\prime into \mathcal{H}$ in line 15. In line 18, we return the optimal feature subset $F^*$.               

\begin{algorithm}[tb]
\small 
\SetKwInOut{Input}{Input}
\Input{ Twitter Data $T$, The Class Values $Y$, Initial Feature Set $F$, }

\SetKwInOut{Output}{Output}
\Output{ Optimal Feature Set $F^*$ }
	
	$F^* \gets \emptyset$; $\mathcal{H} \gets \emptyset$; $\theta^{min} \gets 0$ \tcp*[r]{initialisation}
	
    \While{ $f_i = getNext(F) \neq \emptyset$}
	{	
	   $e.F \gets f_i$;\par
	   $e.\theta \gets \theta(e.F,Y)$;\par
	   $e.\overline{\theta} \gets \frac{IG(Y;e.F) - \log(n) + 1  }{\log(n-1)} + 1$ \tcp*[r]{as per Equation \ref{eq:upperbound}}
	   $\mathcal{H}.push(e)$ \tcp*[r]{insert $e$ into $\mathcal{H}$}
	 }
	 
	 \While{ $\mathcal{H} \neq \emptyset$ }
	 {
	    $e \gets \mathcal{H}.pop()$;\par
	    \If{ $e.\overline{\theta} < \theta^{min}$ }
	    {
	        \textbf{break} \tcp*[r]{as per lemma \ref{lem:admissible}}
	    }
	    \While{ $f_i = getNext(F) \neq \emptyset$ }
	    {
	        $e^\prime.F \gets e.F \cup f_i$ \tcp*[r]{add feature $f_i$ to the feature set}
	        $e^\prime.\theta \gets \theta(e.F,Y)$;\par
	        $e.^\prime\overline{\theta} \gets \frac{IG(Y;e.F) - \log(n) + 1 }{\log(n-1)} + 1$ \tcp*[r]{as per Equation \ref{eq:upperbound}}
	        $\mathcal{H}.push(e^\prime)$ \tcp*[r]{insert $e^\prime$ into $\mathcal{H}$}
	    }
	     $\theta^{min} \gets e.\theta$ \tcp*[r]{update $\theta^{min}$}
	    
	 }
	 $F^* \gets e.F$ \tcp*[r]{optimal feature subset $F^*$ computed}
	\Return{$F^*$};
	\caption{Branch \& Bound Search}
\label{alg:branchbound}
\end{algorithm}


\section{The Greedy-based Sub-optimal Solution}
\label{Sec:GreedySuboptimal}
The branch and bound search algorithm performs an informed search on the feature space of Twitter data to compute the optimal subset $F^* \subset F$, which improves the classification performance. However, the informed branch \& bound algorithm is computationally heavy; thus, it is impractical when the Twitter data arrive dynamically. Therefore, we propose a greedy-based algorithm to improve the performance and to adapt our informed branch \& bound algorithm to a dynamic environment. To this end, we discuss two important properties for the evaluation function $\theta(F,Y)$.

\begin{property}
The function $\theta(F,Y)$ is non-negative, i.e., $\forall F \in \mathcal{F}$, $\theta(F,Y) \geq 0$.  
\end{property}

\begin{property}
The function $\theta(F,Y)$ is non-monotone, i.e., $\exists F_2 \subset F_1 \subset \mathcal{F}$ such that (a) $\theta(F_1,Y) \geq \theta(F_2,Y)$, or (b) $\theta(F_1,Y) < \theta(F_2,Y)$.  
\end{property}

Since $\theta(F,Y)$ is non-monotone, we design a greedy-based branch \& bound algorithm by restricting the search expansion to the most promising features called the seed of features $F^s$. To design the greedy-based algorithm, we discuss two points: (a) we prepare a seed of features $F^s = \{f_1, f_2,...,f_n\}$ based on their single accuracy in the Twitter data classification process and (b) we only expand a search branch when the feature subset score gets improved in each step. Thus, we only expand the search with seed features and eventually increase the number of features. Moreover, we continue the expansion if adding a new feature improves the score of the feature subset. The main steps of our greedy-based algorithms are as follows.

We start with a set of seed features $F^s \subset F$, $\forall f \in F^s$ we generate an entry $e$ by setting the entry's feature set e.F, the entry's score $e.\theta$, and the entry's upper bound $e.\overline{\theta}$. Then we push $e$ into the heap $\mathcal{H}$. In a loop, we retrieve $e$ until $e.\overline{\theta} < \theta^{min}$ and do the following steps: we check if $\exists f^\prime \in F \setminus e.F$ such that $\theta(e.F \cup f^\prime,Y) > \theta(e.F,Y)$ then we add the feature $e.F = e.F \cup f^\prime$ and push $e$ into $\mathcal{H}$. Also, we check if $\exists f^\prime \in e.F$ such that $\theta(e.F \setminus f^\prime,Y) > \theta(e.F,Y)$ then we exclude the feature $e.F = e.F \setminus f^\prime$ and push $e$ into $\mathcal{H}$. At the end of the loop, we return the sub-optimal entry $e$ that contains the sub-optimal feature subset. We define the sub-optimal feature subset as follows:

\begin{definition}\textbf{Sub-optimal Feature Subset}
Given $F^*=\{f_1,f_2,$ $...,f_n\}$ and the classification accuracy $\theta(F^*,Y)$, $F^*$ is sub-optimal feature subset, if (a) $\forall f_i \in F^*$, $\theta(F^*,Y) \geq \theta(F^* \setminus f_i,Y)$, and (b) $\forall f_j \not\in F^*, \theta(F^*,Y) \geq \theta(F^* \cup f_j,Y)$.   
\end{definition}

To further improve the ED-Filter, we employ deep learning to predict the sub-optimal feature subset cardinality and propose a hybrid greedy deep-learning method to skip more unnecessary search expansions.  

\section{The Hybrid Solution}
\label{Sec:HybridDeppLearning}
Although the greedy-based branch \& bound search skips many unnecessary expansions throughout the search process, it is still expensive. That is because the greedy-based search expands many entries with various feature subset sizes. Thus, the search tree becomes extremely deep, which increases the processing complexities. Moreover, the dynamic Twitter data arrives frequently, and we should apply feature selection numerously to extract reliable features for the classification task. Therefore, we propose a hybrid method that combines deep learning with the greedy-based approach to improve performance further. To this end, we design a deep learning model that trains on the Twitter data to predict the suitable feature subset cardinality for running the search. 

\begin{algorithm}[]
\small 
\SetKwInOut{Input}{Input}
\Input{ Twitter Data $T$, The Class Values $Y$, Initial Feature Set $F_D$, Neural Network Model $Model$ }

\SetKwInOut{Output}{Output}
\Output{ Sub-Optimal Feature Set $F^*$ }
	
	$F^* \gets \emptyset$; $\mathcal{H} \gets \emptyset$; $\theta^{min} \gets 0$ \tcp*[r]{initialisation}
	$F^s \gets \{f_1,f_2,...,f_n\}$ \tcp*[r]{seed features}
	$c \gets Model.featureCount(F,Y)$ \tcp*[r]{retrieve the feature subset cardinality}
    \While{ $f_i = getNext(F^s) \neq \emptyset$}
	{	
	   $e.F \gets f_i$;\par
	   $e.\theta \gets \theta(e.F,Y)$;\par
	   $e.\overline{\theta} \gets \frac{IG(Y;e.F) - \log(n) + 1 }{\log(n-1)} + 1$ \tcp*[r]{as per Equation \ref{eq:upperbound}}
	   $\mathcal{H}.push(e)$ \tcp*[r]{insert $e$ into $\mathcal{H}$}
	 }
	 
	 \While{ $\mathcal{H} \neq \emptyset$ }
	 {
	    $e \gets \mathcal{H}.pop()$;\par
	    \If{ $e.\overline{\theta} < \theta^{min}$ }
	    {
	        \textbf{break} \tcp*[r]{as per lemma \ref{lem:admissible}}
	    }
	    \While{ $f_i = getNext(F) \neq \emptyset$ }
	    {
	        \If{ $|e.F| \geq c$ }
	        {
	            \textbf{break}
	        }
	        $e^\prime.F \gets e.F \cup f_i$ \tcp*[r]{add feature $f_i$} 
	        $e^\prime.\theta \gets \theta(e.F,Y)$;\par
	        \If{ $e^\prime.\theta > e.\theta$ }
	        {
	            $e.^\prime\overline{\theta} \gets \frac{IG(Y;e.F) - \log(n) + 1 }{\log(n-1)} + 1$ \tcp*[r]{Equation \ref{eq:upperbound}}
	            $\mathcal{H}.push(e^\prime)$ \tcp*[r]{insert $e^\prime$ into $\mathcal{H}$}
	        }
	        
	    }
	    \While{ $f_i = getNext(e.F) \neq \emptyset$ }
	    {
	        $e^\prime.F \gets e.F \setminus f_i$ \tcp*[r]{exclude feature $f_i$ }
	        $e^\prime.\theta \gets \theta(e^\prime.F,Y)$;\par
	        \If{ $e^\prime.\theta > e.\theta$ }
	        {
	            $e.^\prime\overline{\theta} \gets \frac{IG(Y;e.F) - \log(n) + 1 }{\log(n-1)} + 1$ \tcp*[r]{Equation \ref{eq:upperbound}}
	            $\mathcal{H}.push(e^\prime)$ \tcp*[r]{insert $e^\prime$ into $\mathcal{H}$}
	        }
	        
	    }
	     $\theta^{min} \gets e.\theta$ \tcp*[r]{update $\theta^{min}$}
	    
	 }
	 $F^* \gets e.F$ \tcp*[r]{sub-optimal feature subset $F^*$ computed}
	\Return{$F^*$};
	\caption{Deep Learning Greedy-based Search}
\label{alg:deeplearninggreedy}
\end{algorithm}

The input of the deep learning model is the full feature set with the class values, while the output is the number of selected features. To train the model, we break down the Twitter data into smaller chunks and feed the neural network with these chunks. Also, we feed the neural network with the optimal feature subset cardinality. To find the optimal feature subset size $|F^*|$, we run a comprehensive search on the feature space $F_i$ of each Twitter data chunk $T_i \subset T$ and compute the classification accuracy $\theta(F_i,Y_i)$ for the Twitter data on eating disorder types. Then, we choose the feature subset cardinality that has led to the best classification accuracy to feed the neural network for training purposes. After the model is prepared, we can call the feature count procedure, i.e., $Model.featureCount(F,Y)$, and supply the data, including features $F$ and the class $Y$ to predict the feature subset cardinality. 
Assume we have a table $Dt$ that contains a list of rows that represent the data records, and for each data row the optimal feature cardinality is defined. Then, the neural network $Model$ is trained on the table $Dt$, which predicts the optimal feature cardinality of every input data record. We design a multi-layer perceptron with two hidden layers. The model has one node for each class in the output layer and adopts the softMax activation function. The loss function is the $sparse\_categorical\_crossentropy$. The model is optimized by adopting the $adam$ version of stochastic gradient descent and seeks to minimize the cross-entropy loss.

Algorithm \ref{alg:deeplearninggreedy} presents the hybrid approach. After initialization, we set the feature seed $F^s$ in line 2. Line 3 applies the deep learning model trained on offline data to predict the size of the sub-optimal feature subset $F^*$. In lines 4-8, we initialize the max heap $\mathcal{H}$ with the feature seed. These features are extracted based on their information gain score. In line 10, if $\mathcal{H} \neq \emptyset$, we pop the top element from $e$. Lines 11-12 check the termination condition and finish the search when $e.\overline{\theta} < \theta^{min}$ satisfies. In lines 14-15, if $|e.F| \geq c$, we stop expanding the feature subset because it has already reached the size threshold $c$. Otherwise, we generate the new expanded entry $e^\prime$ by inserting a new feature into $e^\prime.F$ in line 16. Then, we compute the classification accuracy $\theta(e^\prime.F,Y)$ and the upper bound of accuracy $\theta(e^\prime.F,Y)$ for the expanded feature subset $e^\prime.F$ respectively in lines 16-17. In line 18, if the accuracy score of the expanded entry $e^\prime.\theta$ is better than the score of the original entry $e.\theta$ (it means that by adding $f_i$ to $e.F$ accuracy increases), the expanded entry $e^\prime$ is acceptable; hence, we compute the upper bound of accuracy $e^\prime.\overline{\theta}$ and push it into $\mathcal{H}$ in lines 19-20. In line 21, we read a feature from the entry's feature subset $f_i \in e.F$. Then, we remove $f_i$ from the entry feature subset and generate a new shrunk entry $e^\prime$ in line 22. Line 23 computes the classification accuracy $e^\prime.\theta$. In line 24, if the accuracy score of the shrunk entry $e^\prime.\theta$ is better than the score of the original entry $e.\theta$ (it means that by removing $f_i$ from $e.F$ accuracy increases), the new shrunk entry is selected for the search. Therefore, we compute the accuracy upper bound $e^\prime.\overline{\theta}$ and push it into $\mathcal{H}$ in lines 25-26. Finally, in line 29, we return the sub-optimal feature subset $F^*$.                 

\section{Experimental Analysis}
\label{Sec:ExperimentalAnalysis}
In this section, we present experimental results to verify the effectiveness and efficiency of our proposed ED-Filter. All of our algorithms are implemented in Python and the experiments are run on a PC with 1.8 GHz, 8 GB memory, and 64-bit Windows 10. We collected the data from Twitter and from the users who post their Tweets about eating disorders. This study consisted of $11,620$ Pro-ED Twitter accounts that posted using the hashtag $\#proana$ between September 2015 and July 2018. 

\subsection{Data Collection, Cleaning, and Processing}
The Twitter eating disorder raw data was collected by using a streaming API. To collect the ED-related Tweets, we have used a list of hashtags and keywords such as $anamia$, $thinspo$, $fitspo$, and so on. The workflow of our data collection, cleaning, and processing is summarised into three steps: (a) text collection, (b) text cleaning, and (c) topic modeling. 
\begin{itemize}
    \item \textbf{text collection}. We collected 37405 tweets that contain ED-related hashtags and keywords.
    \item \textbf{text cleaning}. We removed the usernames, URLs, stop words, and numbers. Also, we removed non-standard characters such as emojis, tabs, and so on. Then, we lowered the word cases.
    \item \textbf{topic modelling}. We built various CorEx models with a varying number of topics $(n=[2-50])$. Then, we selected the model with the best number of topics. Finally, we extracted the topics by using the best model.   
\end{itemize}

\begin{table*}[t]
\centering
\caption{Extracted Topics on Twitter Eating Disorder Data}
\small
\begin{tabular}{ |l|p{100mm}| }
  \hline
  Topic $\#$ & Topic Keywords \\
  \hline
  Topic 1 & news, skinnyispretty, edproblems, reasonsnottoeat, models, sites, bony, month, banned, mentalillness \\
   \hline
  Topic 2 & wear, kill, honey, clothing, dieting, loss, single, screw, everythingproana, tips  \\
   \hline
  Topic 3 & proanamia, thinsporation, edgoals, disorder, skinnydreams, site, ugly, support, apart, thinspor  \\
   \hline
  Topic 4 & thinspo, ached, water, ice, hunger, calories, people, happiness, heaven, purging  \\
   \hline
  Topic 5 & weightproana, reason, tea, finding, wanted, making, scale, fatproana, inside, sure  \\
  \hline
  Topic 6 & diet, food, abc, hiding, coke, weightloss, dog, reward, loseweight, eatbetter  \\
  \hline
\end{tabular}
\label{tbl:topics}
\end{table*}

\definecolor{colorGreedy}{HTML}{82B0D2}
\definecolor{colorHybrid}{HTML}{FFBE7A}
\definecolor{colorDifference}{HTML}{8ECFC9}

\begin{figure*}[h]
	\centering
    \begin{tikzpicture}[scale=1, every node/.style={scale=1}]
		\fill[colorGreedy] (0,0.1) rectangle (0.8,0.4); 
        \node[anchor=west,black] at (0.9,0.25) {Greedy};

        \fill[colorHybrid] (3,0.1) rectangle (3.8,0.4); 
        \node[anchor=west,black] at (3.9,0.25) {Hybrid};

        \fill[colorDifference] (6,0.1) rectangle (6.8,0.4); 
        \node[anchor=west,black] at (6.9,0.25) {Difference};
	\end{tikzpicture}
	    \subfloat[]{
		\includegraphics[width = 1.5in]{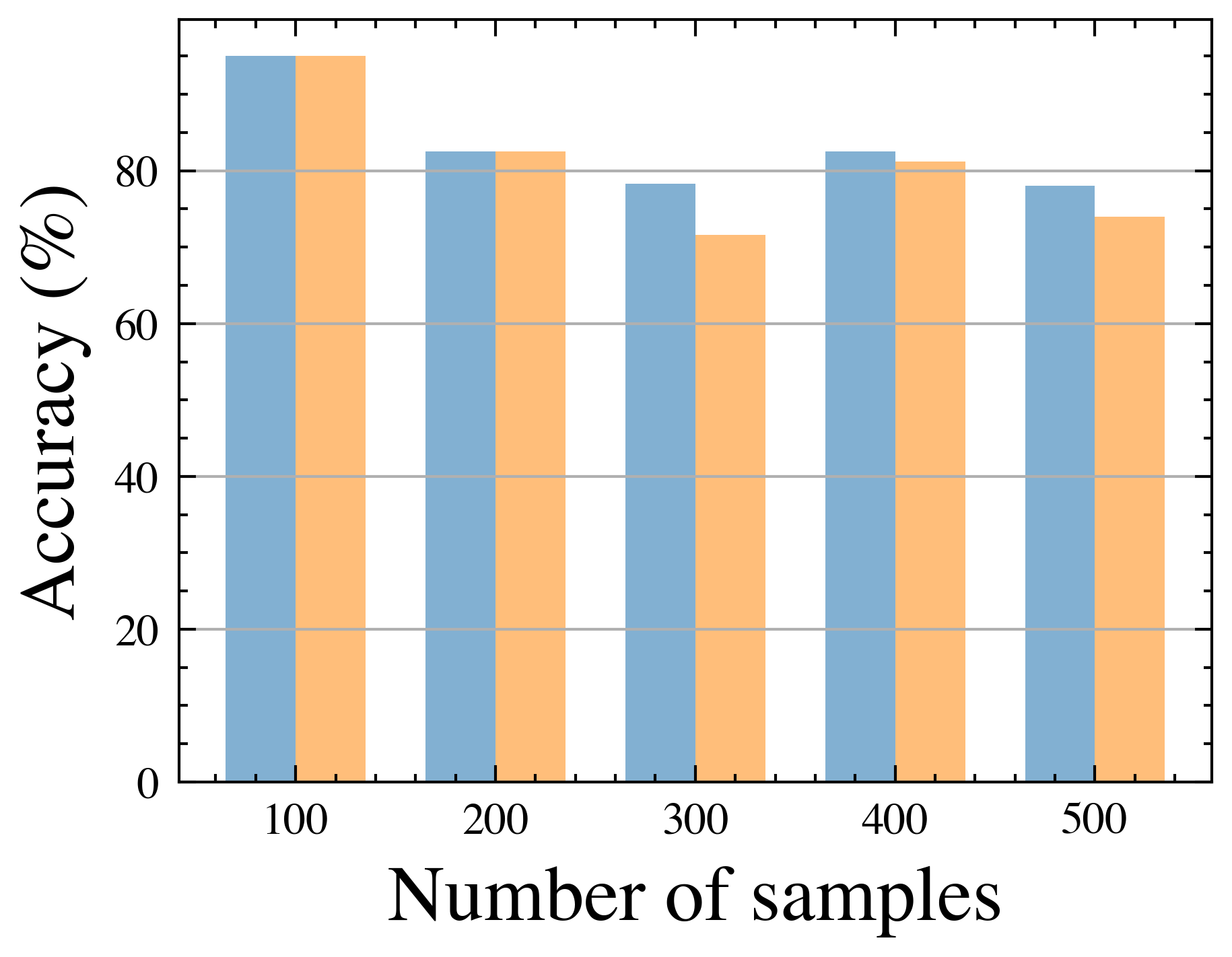}\hspace{0.01cm}
		}
		\subfloat[]{
		\includegraphics[width = 1.5in]{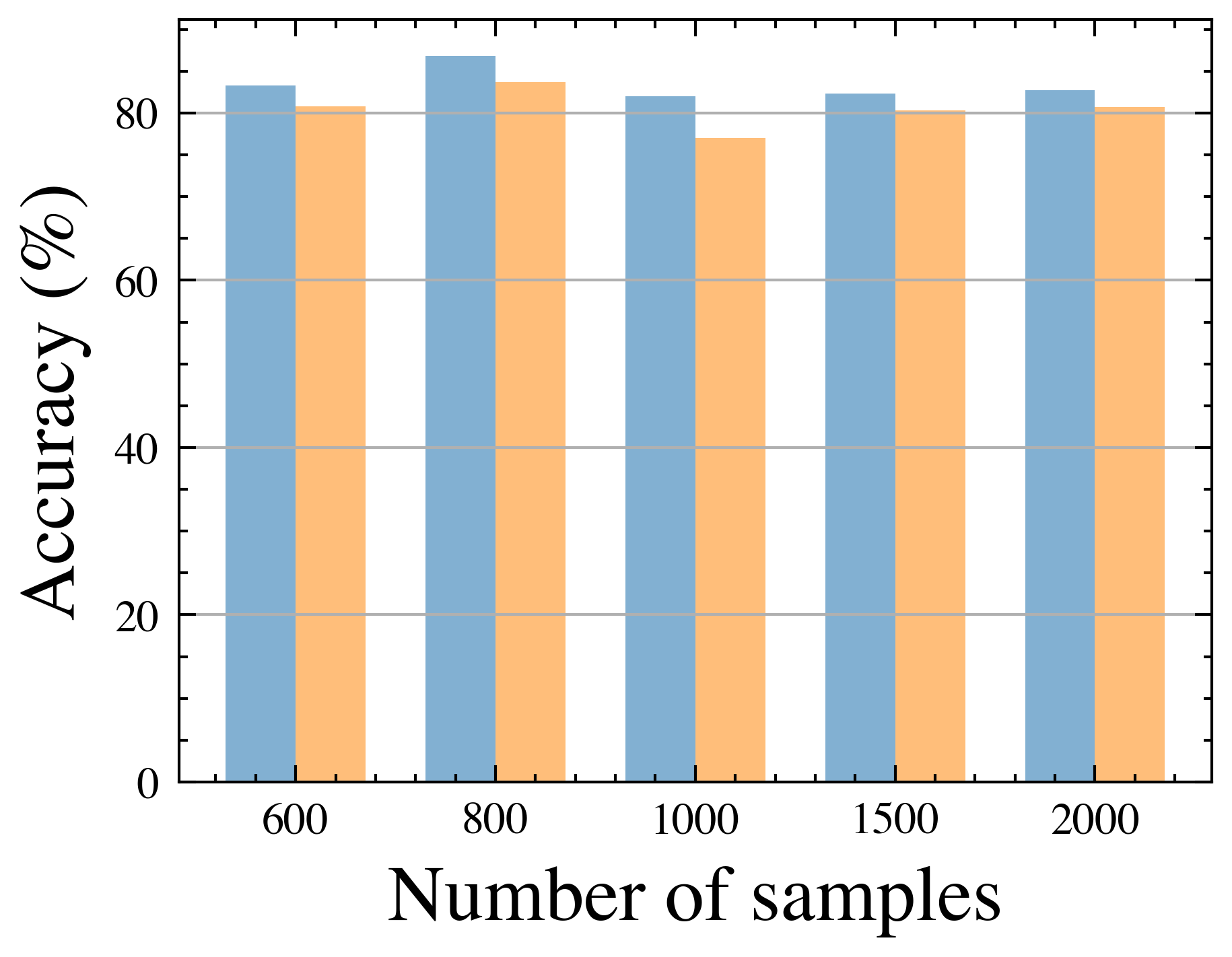}\hspace{0.01cm}
		}
		\subfloat[]{
		\includegraphics[width = 1.5in]{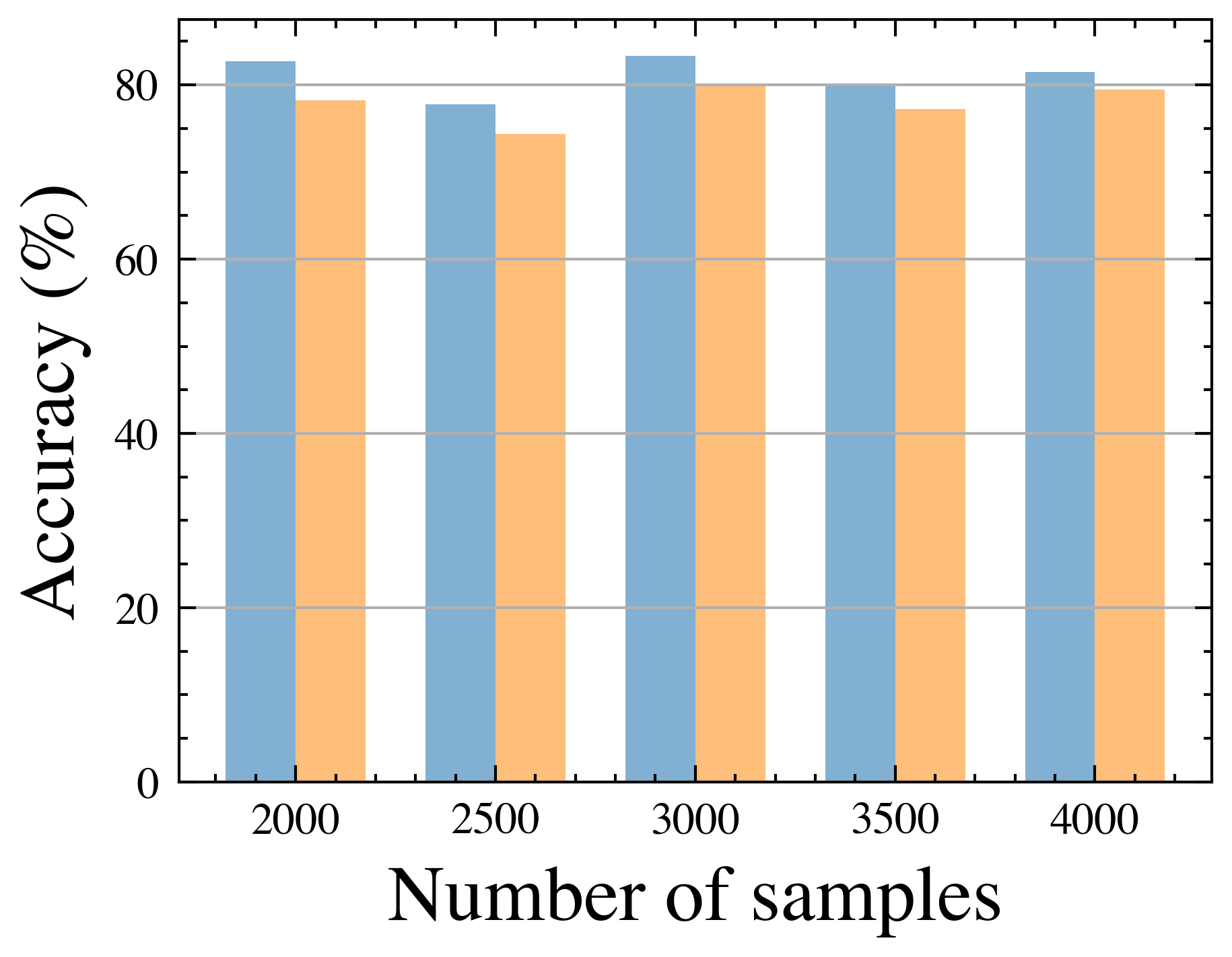}\hspace{0.01cm}
		}\\
		\subfloat[]{
		\includegraphics[width = 1.5in]{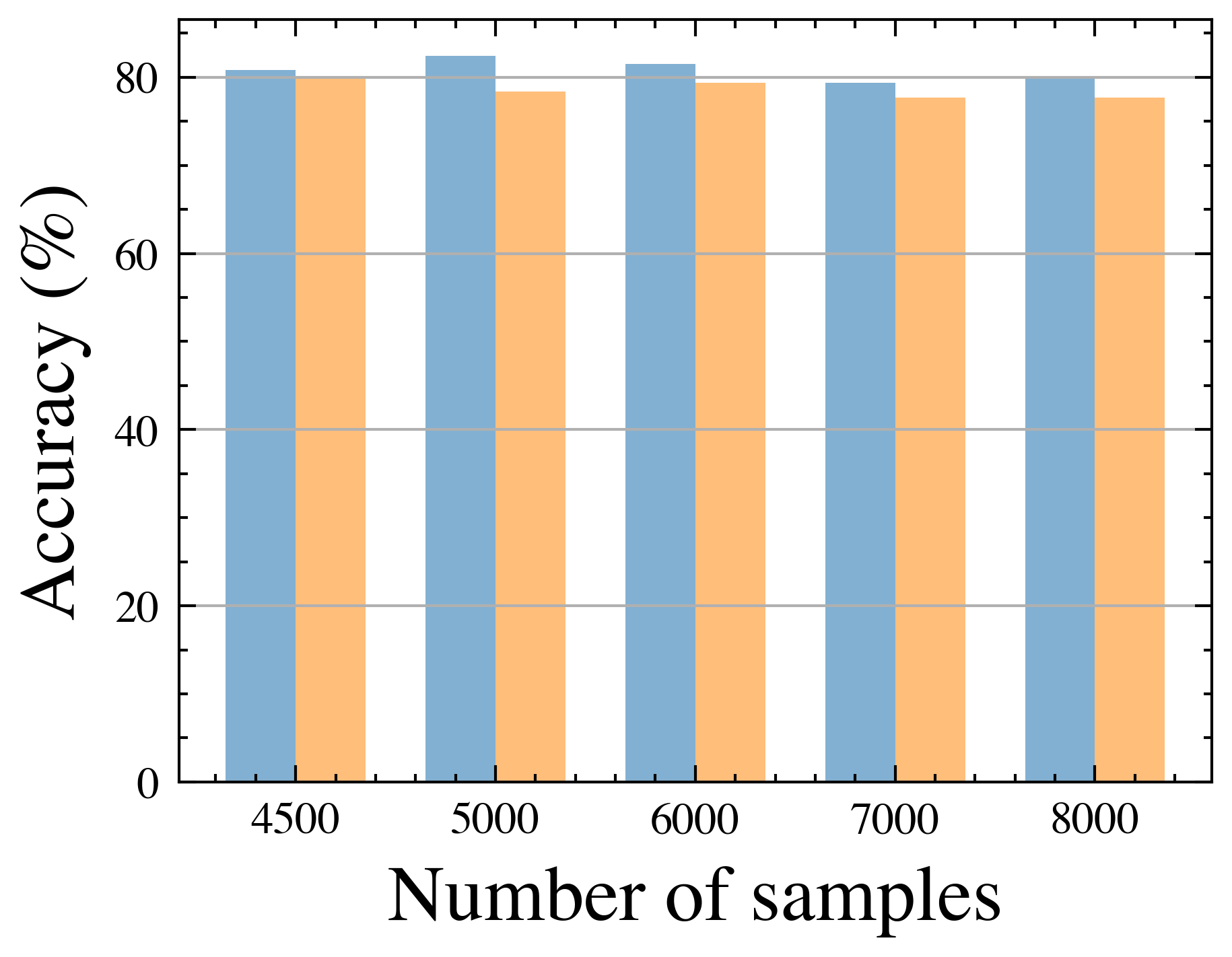}\hspace{0.01cm}
		}
		\subfloat[Worst Case Scenario, $\Delta = 6.5$, $\delta = 3$]{
		\includegraphics[width = 1.5in]{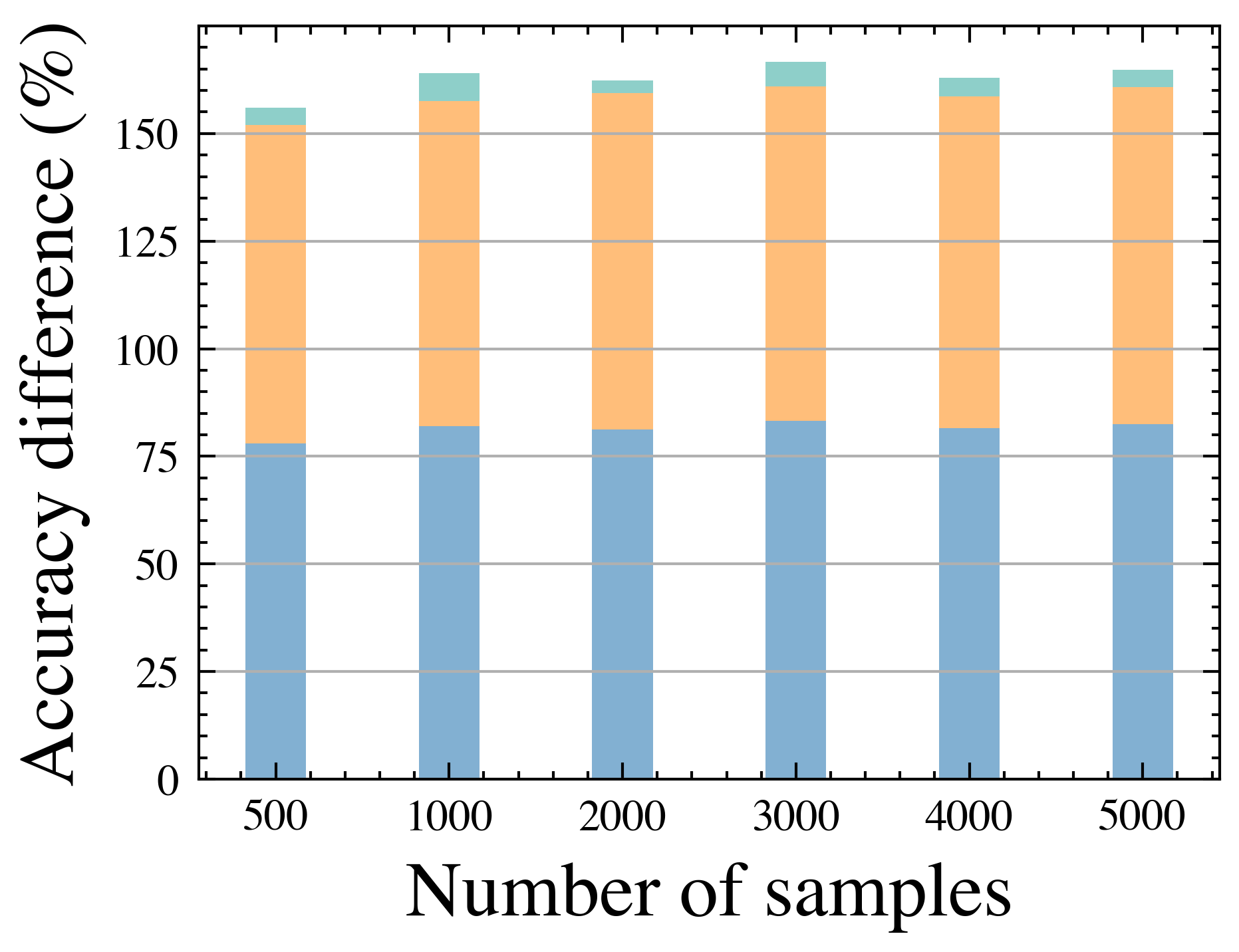}\hspace{0.01cm}
		}
		\subfloat[Best Case Scenario, $\Delta = 4.5$, $\delta = 1$]{
		\includegraphics[width = 1.5in]{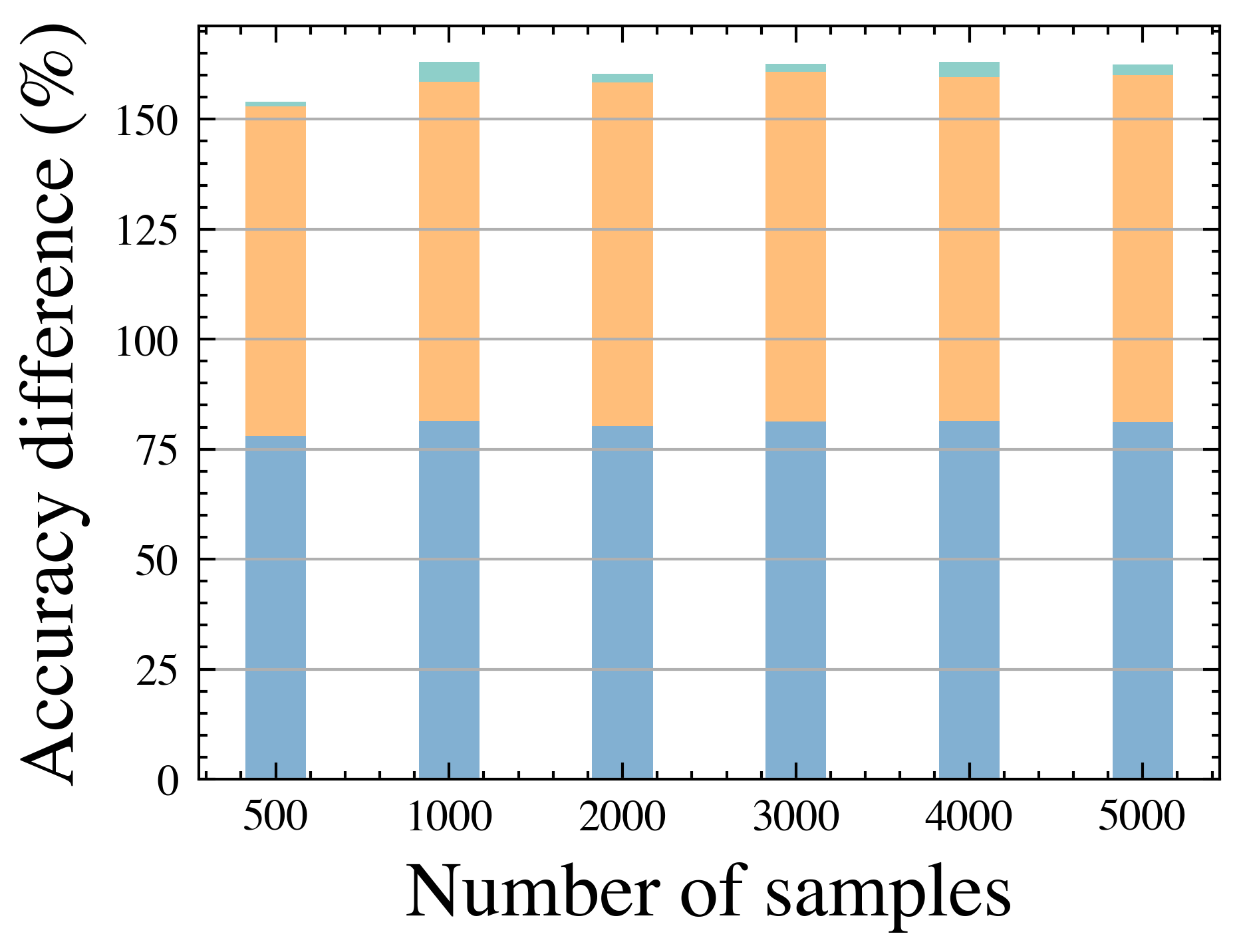}\hspace{0.01cm}
		}
		\caption{Effectiveness Analysis, figure (a), (b), (c), (d) show the effectiveness of greedy and hybird methods, (e) and (f) demonstrate the effectiveness difference between greedy and hybird methods.}
	\label{fig:Effectiveness}
\end{figure*}
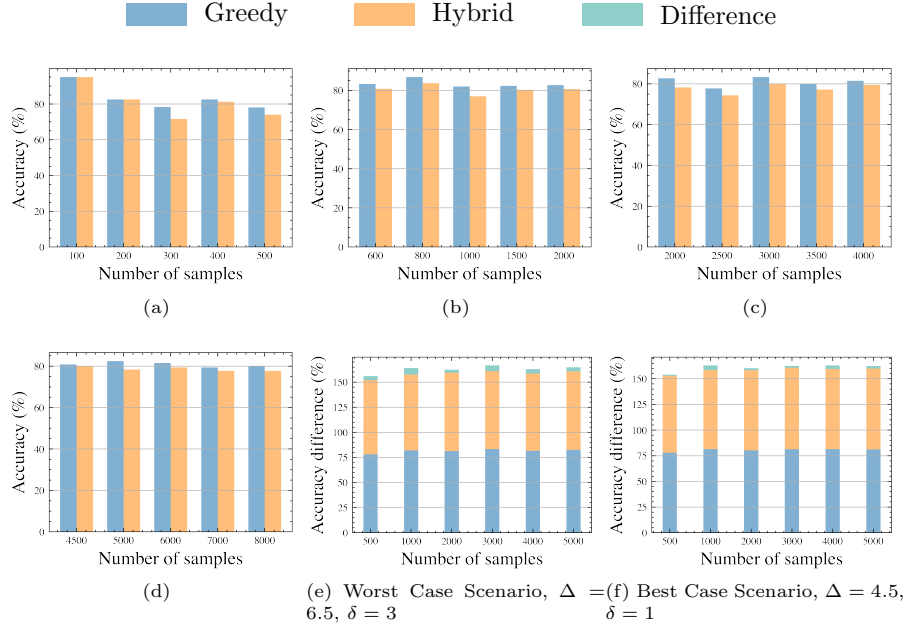

For the text collection, we store the textual records in a MySQL database. We create separate data columns to store the different fields of the Tweets, including the text, creation time, location, count of friends, count of followers, description, and hashtags. We run SQL scripts on the data table to clean the text by removing special characters, tabs, spaces, and emojis.

We utilize Information Gain for our filtering method to rank the features based on the well-established entropy measure. The Information Gain is a suitable choice because it is lightweight and proved effective in terms of ranking the prominent features. Moreover, for our cardinality detector model, we select Relu for our hidden layers activation function because it usually produces good results for the starting point, and we pick up Softmax for our output layer activation function because our problem is a multi-class classification. Moreover, we choose Adam optimizer to speed up the process.

In the topic modeling phase, we aim for extracting and detecting the users' ideas about ED-related issues. Later, we use this information to determine the eating disorders' data features. Therefore, we apply the topic modeling technique to extract the prominent topics within the users' tweet collection. There are several topic modeling methods in the literature. We adopt CorEx statistical modeling since it is a well-established method for topic modeling; moreover, it produces the topics with minimum overlap~\cite{SteegG14, RizviWNV00019}. Table \ref{tbl:topics} presents $6$ topics extracted by CorEx. Each topic in CorEx is represented by a set of keywords. The set of keywords in each topic conveys a particular meaning together.

We determine the features of the Twitter eating disorder data by using the topic keywords and consultation with the eating disorders experts. After consultation with experts, we define 25 features for the Twitter data such as $F$ = \{$ExplicitEDTerms$, $AgainstED$, $ProAnaFamily$, $BodyAndBodyImage$, $Body$ $Weight$, $FoodAndMeals$, $EatORAte$, $Exercise$, $Binge$, $Fitspo$, $Beauty$, $Xmas$, $Summer$, $Winter$, $Halloween$, $BodyParts$, $Bullying$, $DomesticViolence$, $MentalHealth$, $Depressed$, $Suicide$, $Accessory$, $ContemporaryBehavior$, $Thinspo$, $CaloricRestriction$\}. Each feature $f \in F$, is represented by a set of keywords from the topic modelling process and consultation with eating disorders experts. For instance, the feature $BodyImage$ is represented by the following keywords: $overweight, obese, fatty, skinnier, skeleton$. To generate the eating disorder classification data for each feature, we compute the total number of their corresponding keywords appearance within the Tweet collection. Thus, for each user, we create a row that contains $25$ features (columns). Each feature (column) represents the aggregated count of its corresponding keywords that appeared in the user's tweet collection.

\subsection{Effectiveness Assessment}
In this section, we compare the effectiveness of our proposed ED-Filter algorithm. We use the classification accuracy of the reduced (filtered) eating disorder data to verify the effectiveness of the feature selection techniques. 

Figure \ref{fig:Effectiveness} presents the effectiveness results. Figures \ref{fig:Effectiveness}.a to \ref{fig:Effectiveness}.d compare the classification accuracy between the greedy and hybrid methods by varying the number of samples in the Twitter eating disorders data. The accuracy of the greedy method is equal to or better than the hybrid method in all cases. That's because the greedy method performs an informed search on the data feature space and computes the sub-optimal feature space; however, the hybrid method performs the search on a specific number of feature spaces with certain cardinality, which may skip the sub-optimal feature space. Furthermore, the results verify that the hybrid method achieves comparable effectiveness with the greedy method. This indicates that the hybrid method successfully skips many non-promising solutions, which is capable of finding solutions that are very close to the greedy method in terms of effectiveness. 

Figures \ref{fig:Effectiveness}.e and \ref{fig:Effectiveness}.f present the difference in classification accuracy between greedy and hybrid methods in the worst and best-case scenarios, respectively. The maximum and minimum accuracy differences between the greedy and hybrid methods are denoted by $\Delta$ and $\delta$, respectively. In Figure \ref{fig:Effectiveness}.e, the maximum difference $\Delta=6.5$ and the minimum difference is $\delta=3$ in the worst-case scenario. In this experiment, we run each feature selection method a couple of times and compute the maximum difference between the two methods. Conversely, in Figure \ref{fig:Effectiveness}.f, we compute the minimum difference between the two methods. The maximum difference is $\Delta=4.5$ while the minimum difference is $\delta=1$ in the best-case scenario. On average, the accuracy difference is below five percent between the two methods, which verifies that the hybrid method is comparable to the greedy method in terms of effectiveness.               

\definecolor{EDFilterColor}{HTML}{8ECFC9}
\definecolor{InfoGainFilterColor}{HTML}{FFBE7A}
\definecolor{ReliefFilterColor}{HTML}{FA7F6F}
\definecolor{WrapperBestFirstColor}{HTML}{82B0D2}
\definecolor{WrapperForwardColor}{HTML}{2878b5}

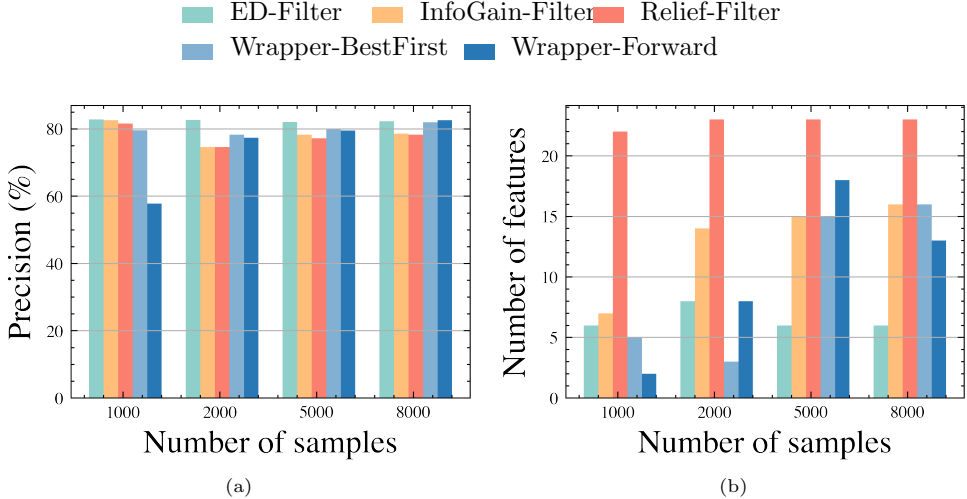
\begin{figure}
\hspace{-0.1cm}
	\centering

    \begin{minipage}{0.9\linewidth}
        \centering
        \begin{tikzpicture}
            \newcommand{\blockwidth}{0.4}  
            \newcommand{\blockheight}{0.2}  
            \newcommand{\blockspacing}{2.3}  
            \newcommand{\linespacing}{0.3}  

            \fill[EDFilterColor] (0,0) rectangle (\blockwidth,-\blockheight);
            \node[anchor=west] at (\blockwidth + 0.1, -0.25*\blockheight) {ED-Filter};
            
            \fill[InfoGainFilterColor] (0+\blockwidth+\blockspacing-0.2,0) rectangle (\blockwidth*2+\blockspacing-0.2,-\blockheight);
            \node[anchor=west] at (\blockwidth*2 + 0.1 +\blockspacing-0.2, -0.25*\blockheight) {InfoGain-Filter};
            
            \fill[ReliefFilterColor] (0+2*\blockwidth+2*\blockspacing,0) rectangle (\blockwidth*3+2*\blockspacing,-\blockheight);
            \node[anchor=west] at (\blockwidth*3 + 0.1 +2*\blockspacing, -0.25*\blockheight) {Relief-Filter};

            \fill[WrapperBestFirstColor] (0,-\blockheight-\linespacing) rectangle (\blockwidth,-2*\blockheight-\linespacing);
            \node[anchor=west] at (\blockwidth + 0.1, -1.25*\blockheight-\linespacing) {Wrapper-BestFirst};
            
            \fill[WrapperForwardColor] (0+\blockwidth+\blockspacing+1,-\blockheight-\linespacing) rectangle (\blockwidth*2+\blockspacing+1,-2*\blockheight-\linespacing);
            \node[anchor=west] at (\blockwidth*2 + 0.1 +\blockspacing+1, -1.25*\blockheight-\linespacing) {Wrapper-Forward};
        \end{tikzpicture}
    \end{minipage}
    
	    \subfloat[]{\label{fig:Comparison}
		\includegraphics[width = 0.48\linewidth]{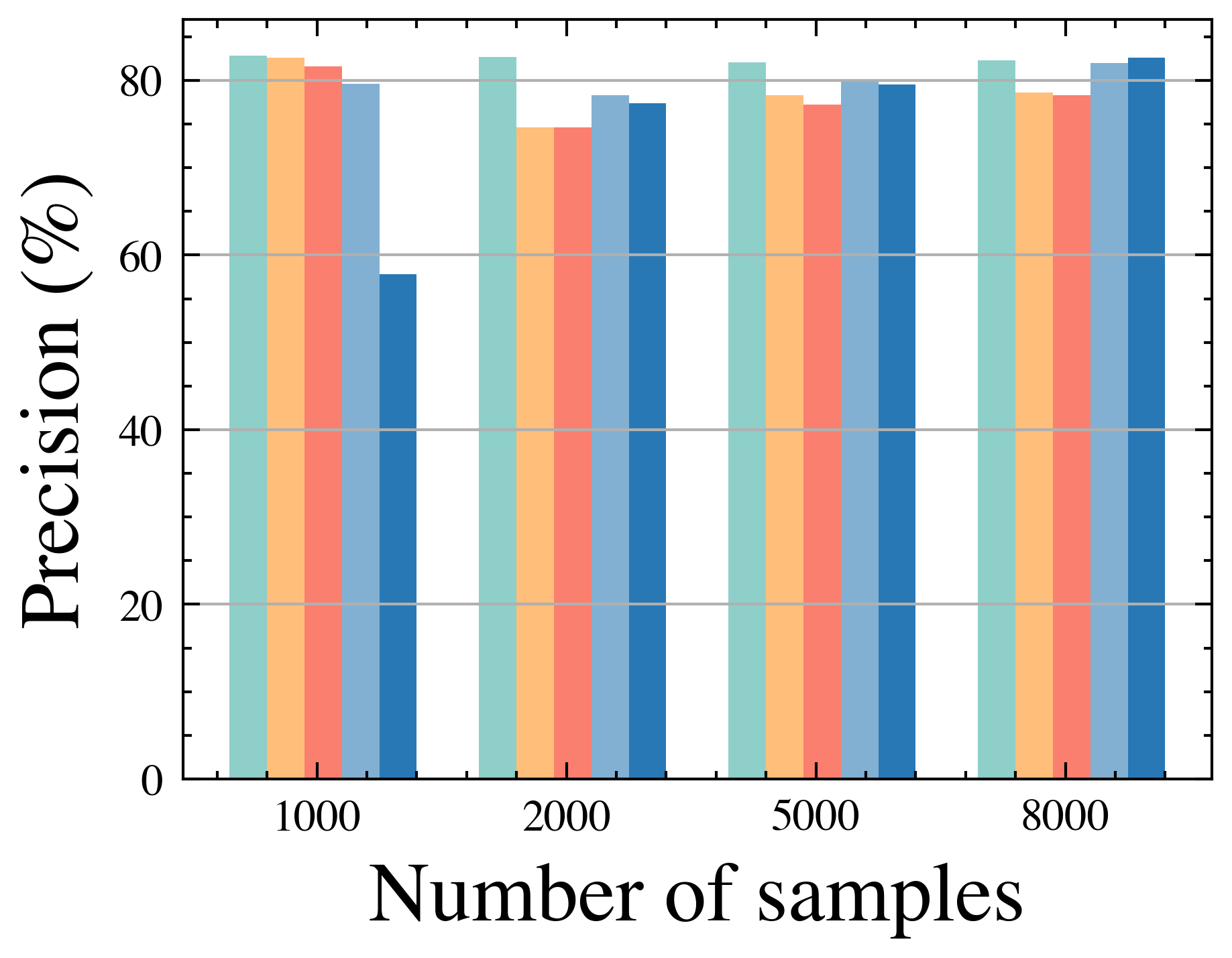}
		}
		\subfloat[]{\label{fig:Cardinality}
		\includegraphics[width = 0.48\linewidth]{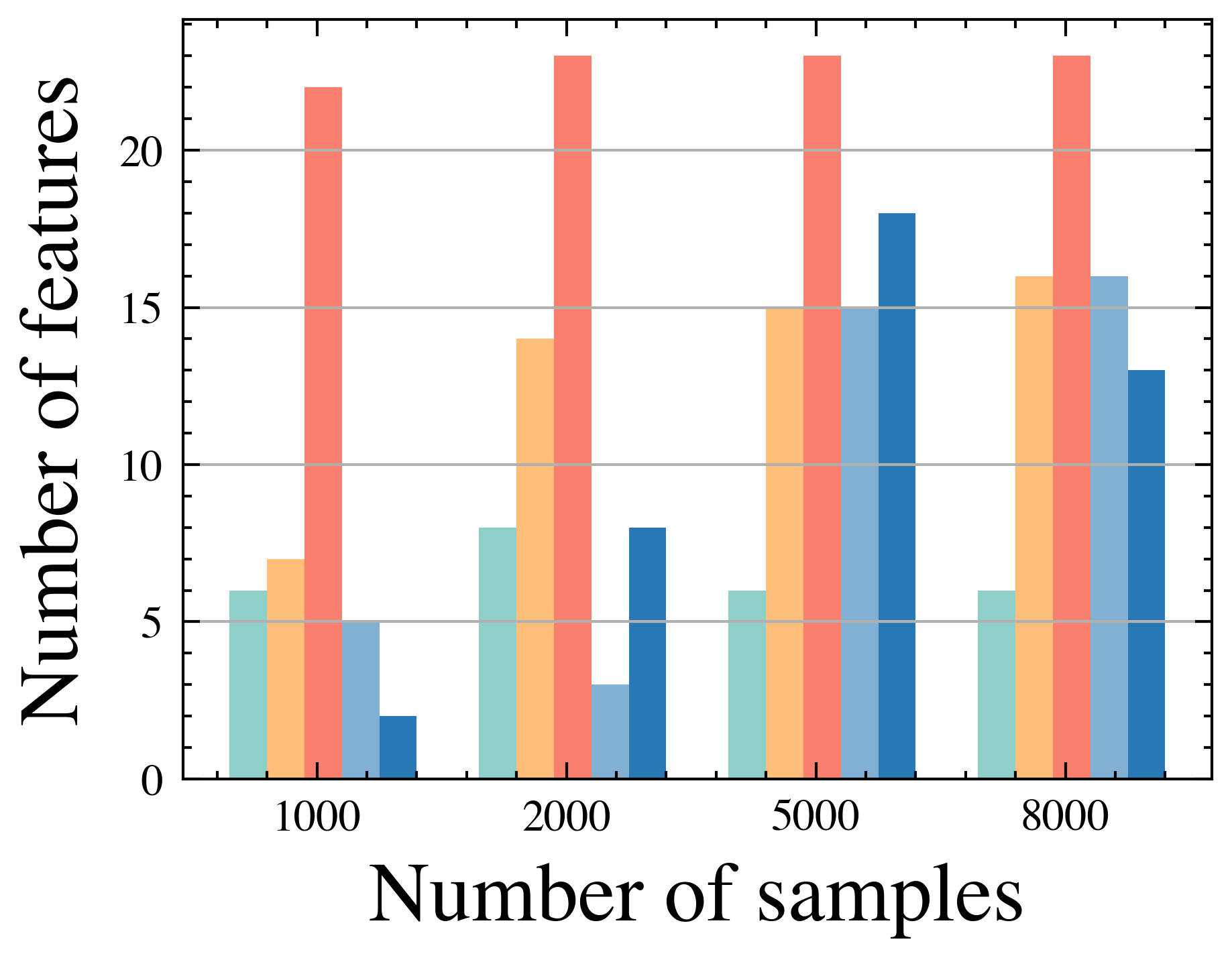}
		}
		
		\caption{ED-Filter versus Other Methods.}
	\label{fig:EffectivenessComparison}
\end{figure}

\subsection{Comparison with Other Methods}
We compare ED-Filter with a couple of established feature selection methods that are categorized as filters or wrappers. The comparison results are presented in Figure \ref{fig:EffectivenessComparison}. We choose the well-established Information Gain~\cite{Shang12} and Relief feature filtering~\cite{URBANOWICZ2018189} methods that are called \textit{InfoGain-Filter}, and \textit{Relief-Filter} in Figure \ref{fig:EffectivenessComparison}. Also, we choose two wrappers~\cite{KohaviJ97, BERMEJO201235} that utilize the best first search and linear forward selection to perform the subset evaluation, respectively. We call them \textit{Wrapper-BestFirst}, and \textit{Wrapper-Forward} in Figure \ref{fig:EffectivenessComparison}. To make a fair comparison, the wrappers employ the Multinomial Naive Bayes classifier to evaluate the feature subsets as we employ this classifier in our ED-Filter.

Figure \ref{fig:EffectivenessComparison}.a compares the precision of various feature selection methods with different data sizes. From the Figure, when the data size equals $1000$, the best precision is for ED-Filter, which is around 83$\%$. Then, the InfoGain-Filter achieves the second-highest precision. However, the wrappers are below the filters and ED-Filter when the data size is $1000$. When we increase the data size to $2000$ samples, the wrappers' precision improves. Conversely, the precision of the filter deteriorates and gets below 75$\%$ while the ED-Filter remains stable in terms of accuracy. When we further increase the data size to $5000$ samples, the wrappers exhibit a better precision for feature selection than the filters. However, the ED-Filter achieves the best precision. That is because the ED-Filter effectively performs an informed search on the feature space and quickly finds the sub-optimal solution. However, the filters utilize the internal data statistics to compute the feature subset without evaluating the merit of subsets against a learning model. Finally, when the data size gets to $8000$ samples, we observe that our ED-Filter achieves a comparable precision with the previous rounds, which are above 82$\%$. This verifies that ED-Filter is a scalable feature selection method and that the data size has the minimum impact on its accuracy. However, the feature filtering methods' precision is negatively affected when the data size grows. On the other hand, the wrappers' precision is negatively affected when the data size is small. That is because the wrappers employ a learning model to assess the feature subsets, and when the data size grows, they compute the more effective solution. To sum up, the ED-Filter is the only stable feature selection method on all data sizes. It achieves the best precision in most cases. The filters' precision is better than wrappers when the data size is small. Finally, the wrappers' precision gets closer to ED-Filter when the data size grows; however, they are more costly than the ED-Filter.

Figure \ref{fig:EffectivenessComparison}.b presents the number of selected features for each method when we vary the data size. From the Figure, the ED-Filter chooses the minimum number of features in comparison with other methods. This verifies that the informed search and the termination condition effectively break the feature selection process quickly. The quick termination property has made the ED-Filter more scalable and practical for dynamic Twitter eating disorder data classification. The Relief-Filter has the maximum number of selected features among others, which is around $23$ features in most cases. The number of selected features in the wrappers increases to above $12$ when the data size grows to $5000$ samples. This verifies that despite their high accuracy in big data sizes, the wrappers do not generate a scalable solution.   
\begin{table*}[t]
\centering
\caption{Comparison of Dimensionality Reduction Methods}
\small
\begin{tabular}{ |l|l|l|l|l| }
  \hline
  Method & class1 & class2 & class3 & class4 \\
  \hline
  SVD (85\%) & 0.79 & 0.4 & 0.41 & 0.35 \\
   \hline
  SVD (65\%) & 0.79 & 0.39 & 0.38 & 0.3  \\
   \hline
  SVD (45\%) & 0.77 & 0.35 & 0.36 & 0.4   \\
   \hline
  LDA (85\%) &  0.81 & 0.5 & 0.55 & 0.42 \\
   \hline
  LDA (65\%) &  0.81 & 0.48 & 0.62 & 0.35 \\
   \hline
  LDA (45\%) &  0.81 & 0.44 & 0.62 & 0.32 \\
   \hline
  ED-Filter  &  0.87 & 0.6 & 0.72 & 0.58 \\
  \hline
\end{tabular}
\label{tbl:comparison}
\end{table*}

In this section, we evaluate our proposed ED-Filter against established dimensionality reduction techniques, including Singular Value Decomposition~(SVD)~\cite{Jianwen2024} and Linear Discriminant Analysis~(LDA)~\cite{Li2024}. We perform our experiments by varying the variance parameters for SVD and LDA to values of $85\%$, $65\%$, and $45\%$. Table \ref{tbl:comparison} presents the comparison outcomes among ED-Filter, SVD, and LDA. From the table, we observe that all methods exhibit relatively acceptable performance concerning the precision of class 1. This can be attributed to the high number of records in class 1, allowing the dimensionality reduction techniques to effectively extract patterns from the majority class during the reduction process. Furthermore, LDA generally outperforms SVD in most cases because it employs a supervised approach to the reduction process. In contrast, the precisions for class 2 and class 4 are lower than those of classes 1 and 3, primarily due to the smaller number of records available. Our proposed ED-Filter achieves better precision across all classes compared to both SVD and LDA. This improvement is a result of ED-Filter's informed search strategy, which combines wrapper methods with deep learning techniques, utilizes a more systematic fine-grained feature selection process, and selects the most reliable features in a supervised manner.

\subsection{Scalability Analysis}

We further test the scalability of the branch and bound search method with the hybrid method in ED-Filter. The scalability of branch and bound search algorithm is illustrated in Figure \ref{fig:BBoundPerformance}. In Figure \ref{fig:BBoundPerformance}.a, we restrict the number of features to $6$. We observe that the execution time grows fast when there are $5$ and $6$ features. Figure \ref{fig:BBoundPerformance}.b presents the results when the feature size reaches up to $9$ features. The branch and bound search method exhibits a sharp jump when the number of features rises above $6$. In each step from $6$ to $9$, the execution time increases sharply. This verifies that despite the informed search and the termination condition in the branch and bound search method, this method is not scalable. When the number of features grows, the execution time grows exponentially, which makes branch and bound search method impractical for the eating disorder classification task.   

\begin{figure}[h]
   \centering
   \includegraphics[scale = 0.5]{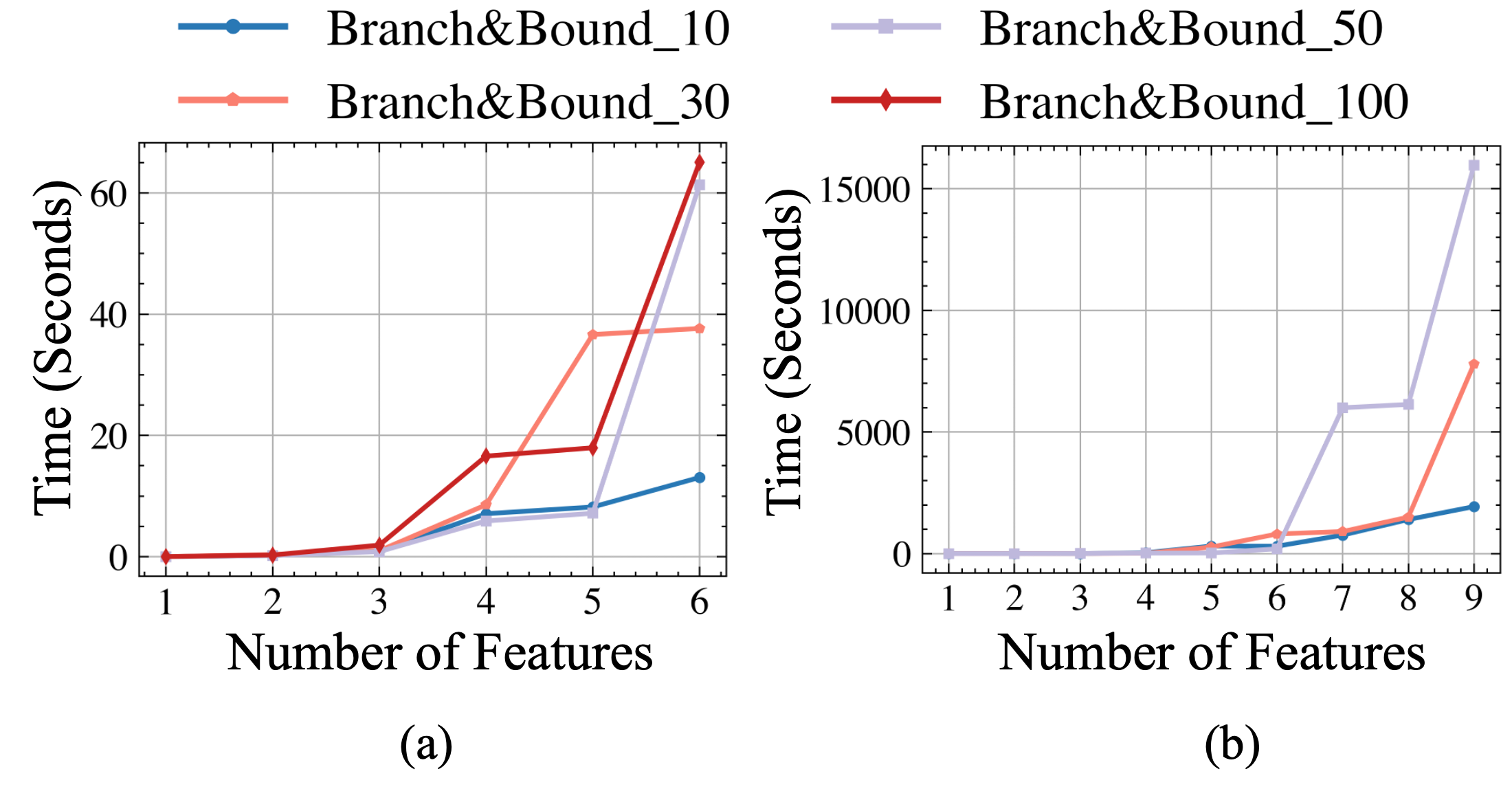}
    \caption{Scalability of hybrid greedy-deep learning method.}
    \label{fig:BBoundPerformance}
\end{figure}

\begin{figure}[h]
   \centering
   \includegraphics[scale = 0.5]{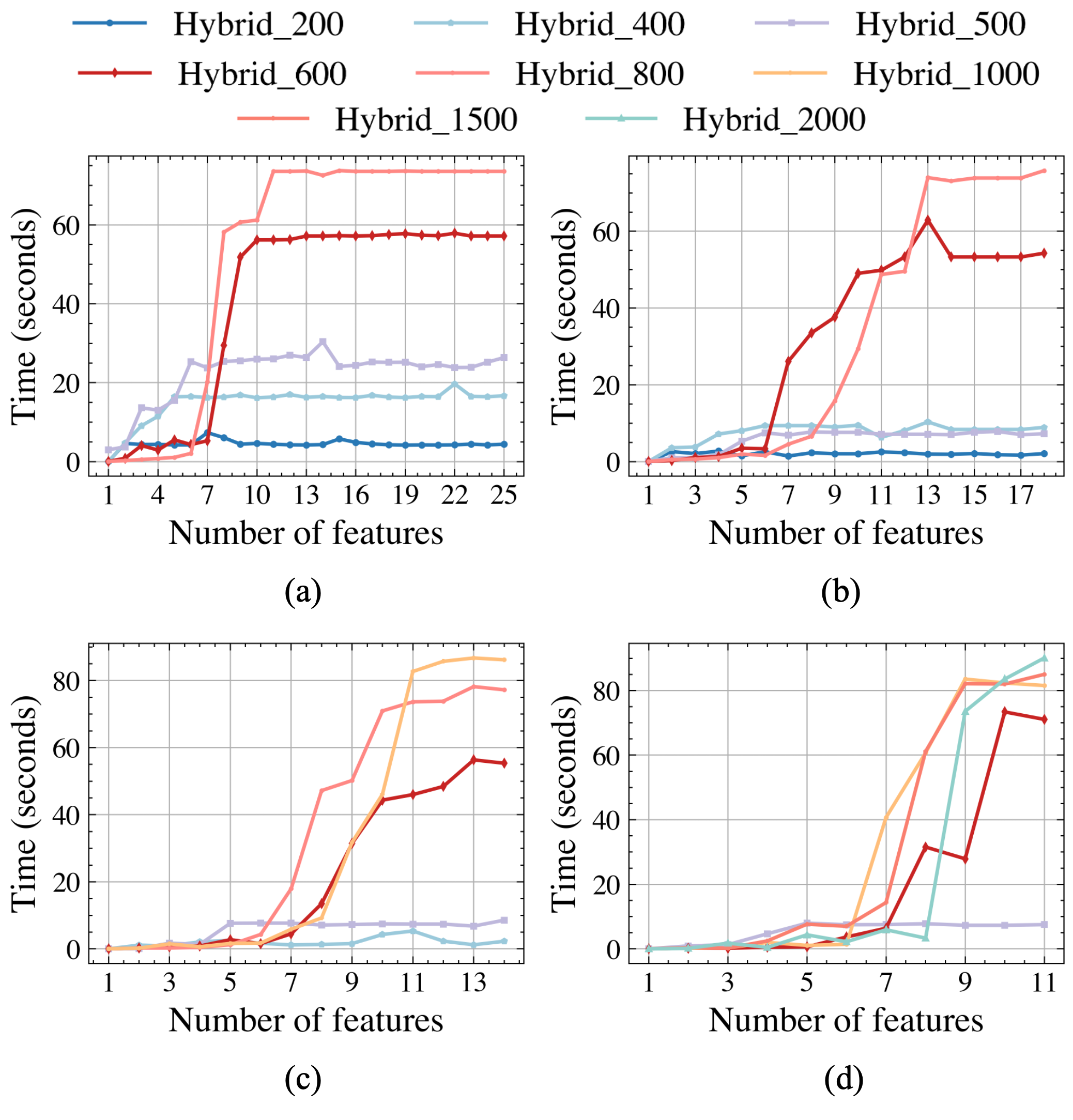}
    \caption{Scalability of hybrid greedy-deep learning method.}
   \label{fig:HybridPerformance}
\end{figure}

Figure \ref{fig:HybridPerformance} presents the scalability of the hybrid method. We diversified the hybrid method by varying the number of samples, i.g., $Hybrid_{200}$ denotes that the hybrid method is applied to the data of 200 samples. Furthermore, we vary the information gain threshold in a couple of steps, and select the features that their information gain scores are above the threshold as illustrated in Figure \ref{fig:HybridPerformance}.a to Figure \ref{fig:HybridPerformance}.d. In Figure \ref{fig:HybridPerformance}.a, 25 features are eligible since their information gain scores are above the threshold. Figure \ref{fig:HybridPerformance}.b presents the results with 18 features because the threshold is bigger than Figure \ref{fig:HybridPerformance}.a. Considering Figure \ref{fig:HybridPerformance}.a, the execution time is steady for $Hybrid_{200}$ and $Hybrid_{400}$. That's because the number of samples is relatively small; therefore, when the number of features increases, it does not affect the performance considerably. However, when the number of samples increases above $500$, the execution time is heavily impacted. For example, the execution time in $Hybrid_{500}$ exhibits a sharp jump when the number of features increases above $6$. Afterward, it gets stable. We see a big jump in the execution time in $Hybrid_{600}$ and $Hybrid_{800}$ when the number of features increases to $8$ and $9$, respectively. After the sharp jump, the execution time becomes stable in the three methods. The jump in the execution time verifies that at a certain point (when the number of features increases to a certain number), the search process looks into an extensive number of solutions and computes their accuracy to find the best solution. However, the sudden jump becomes stable since the hybrid method successfully terminates the search by checking the termination condition so that many non-promising solutions are ignored.

In Figure \ref{fig:HybridPerformance}.c, there are $14$ features because we set the information gain threshold to a large number. We assess the scalability of the hybrid method with large sample sets, i.e., $\{400,500,600,800,1000\}$. Clearly, the big jump in the execution time occurs when the number of features rises to between $7$ and $10$. Furthermore, we observe that when the number of initial features in the feature selection process is not over $15$, the execution time does not exceed $90$ seconds. In Figure \ref{fig:HybridPerformance}.d, we assess the hybrid method by applying the data with the following sample sizes: $\{500,800,1000,1500,2000\}$. We record the execution time for each step while the number of features increases to $11$. The big jump occurs when the size of the feature set is between $6$ and $9$. The execution time does not exceed $90$ seconds, even when the number of data samples gets bigger. That's because the number of features is restricted, and the hybrid method is scalable that it is able to perform the feature selection in a quick response time.

\section{Conclusion}
\label{Sec:Conclusion}
We proposed ED-Filter to pre-process the high-dimensional Twitter eating disorder data and to detect the reliable features. The proposed method reduces the Twitter data dimensions, which improves the classification accuracy of eating disorder data. ED-Filter adopts an iterative filtering approach to perform an informed search on the data space and to find the most relevant dimensions for the classification task. ED-Filter efficiently skips unnecessary computations by setting up a termination condition. To further improve the filtering performance, specifically for dynamically produced data on Twitter, we proposed a hybrid greedy-based deep learning approach that effectively computes the sub-optimal feature subset from the data source. Finally, our experimental results verify the effectiveness and efficiency of the proposed methods. The results prove that the hybrid approach is a practical way of Twitter data pre-processing for the eating disorder classification when the data is produced dynamically and modified frequently. The effectiveness results verify that the hybrid technique is comparable to its baseline and greedy versions. We have selected Twitter as it is an established benchmark for eating disorder studies, and many ED advocates and their networks are deeply active on this platform. However, it is worth considering other social media platforms to analyze the behavior of ED advocates and develop early intervention models for them. Moreover, analyzing the emojis and images and combining their meaning with the textual messages would be beneficial to improve the intervention task. These are intriguing
research directions that we are going to focus in the future.

\vspace{-0.3cm}

\section*{Acknowledgment}
This work is supported by Medical Research Future Fund (Grant Number: MRFF APP1179321) and the Researchers Supporting Project Number (RSPD2025R681) King Saud University, Riyadh, Saudi Arabia.

\section*{Author contributions}
M.N., A.T., and F.X. conceived and designed the research; M.N. and S.S. collected and processed the data; Z.C., O.A., and S.F.R. performed the analysis; all authors wrote and/or revised the paper. 

\section*{Ethics statement}
Our data collection meets the ethical standards and was approved by Swinburne University Human Research Ethics Committee (SUHREC) with the following reference 20190402-1922. Personal information was removed and anonymized.

\section*{Code availability statement}
The implementation code is available at GitHub repository: https://github.com/CapstonProjectsMehdi/ED-Filter-Project  

\section*{Data availability statement}
Data is available upon request.

\section*{Competing interests}
Feng Xia is an Associate Editor of Artificial Intelligence Review.

\bibliographystyle{unsrt}
\bibliography{sn-bibliography}

\end{document}